\documentclass{article}

\usepackage{graphicx}
\usepackage{amssymb}
\usepackage{slashbox}
\usepackage{adjustbox}
\usepackage{tabularx}
\usepackage{scalerel}

\RequirePackage[OT1]{fontenc}
\RequirePackage{amsmath}
\RequirePackage[colorlinks,citecolor=blue,urlcolor=blue]{hyperref}
\usepackage[ruled]{algorithm2e}
\usepackage{algorithmic}
\usepackage{graphicx,psfrag,epsf}
\usepackage{epstopdf}
\usepackage{subfloat}
\usepackage{rotating}
\usepackage{subfigure}
\usepackage{longtable}
\usepackage{engord}
\usepackage{url}
\usepackage{enumerate}
\usepackage{color}
\usepackage{layout}
\usepackage{amssymb}
\usepackage{float}
\usepackage{hhline}
\usepackage{version}
\usepackage{fancyheadings}
\usepackage{appendix}
\usepackage{epstopdf}
\usepackage{amsthm}
\usepackage{multirow}

\newcommand {\ctn}{\cite} 

\newtheorem{definition}{Definition}[section]
\newtheorem{definition*}{Definition}

\newtheorem{theorem}{Theorem}[section]
\newtheorem{remark}{Remark}

\newtheorem{observation}{Observation}

\newcommand{\bed}{\begin{definition}}
\newcommand{\eed}{\end{definition}}
\newcommand{\beq}{\begin{equation}}
\newcommand{\eeq}{\end{equation}}

\newcommand{\bfI}{\mathbf{I}}

\newcommand{\bfX}{\mathbf{X}}
\newcommand{\bfL}{\mathbf{L}}
\newcommand{\bfC}{\mathbf{C}}
\newcommand{\bfE}{\mathbf{E}}
\newcommand{\bfV}{\mathbf{V}}
\newcommand{\bfW}{\mathbf{W}}
\newcommand{\bfP}{\mathbf{P}}
\newcommand{\bfx}{\mathbf{x}}
\newcommand{\bfy}{\mathbf{y}}
\newcommand{\bfZ}{\mathbf{Z}}
\newcommand{\bfc}{\mathbf{c}}
\newcommand{\bfF}{\mathbf{F}}
\newcommand{\bfQ}{\mathbf{Q}}
\newcommand{\bfe}{\mathbf{e}}
\newcommand{\bfS}{\mathbf{S}}
\newcommand{\bfg}{\mathbf{g}}

\newcommand{\mbfX}{\bfC_\bfX}
\newcommand{\mbfL}{\bfC_\bfL}
\newcommand{\mbfE}{\bfC_\bfE}
\newcommand{\mbfS}{\bfC_\bfS}
\newcommand{\hP}{\hat{\mathbf{P}}}

\DeclareMathOperator*{\argmin}{\rm argmin}

 \graphicspath{ {Fig/} }
\begin{document}

\title{Manifold Fitting in Ambient Space} 
\maketitle

\author{\noindent Zhigang Yao\\
Department of Statistics and Applied Probability  \\ 21 Lower Kent Ridge Road\\ 
National University of Singapore, Singapore 117546  \\ 
\noindent  email: \texttt{zhigang.yao@nus.edu.sg}\\  

\noindent Bingjie Li\\
Department of Statistics and Applied Probability  \\ 21 Lower Kent Ridge Road\\ 
National University of Singapore, Singapore 117546  \\ 
\noindent  email: \texttt{bjlistat@nus.edu.sg}\\

\noindent Wee Chin Tan\\
Department of Statistics and Applied Probability  \\ 21 Lower Kent Ridge Road\\ 
National University of Singapore, Singapore 117546  \\ 
\noindent  email: \texttt{statanw@nus.edu.sg}}

\vspace{0.1 in}
\begin{center}
\textbf{Abstract}
\end{center}

Modern sample points in many applications no longer comprise real vectors in a real vector space but sample points of much more complex structures, which may be represented as points in a space with a certain underlying geometric structure, namely a manifold. Manifold learning is an emerging field for learning the underlying structure. The study of manifold learning can be split into two main branches: dimension reduction and manifold fitting. With the aim of combining statistics and geometry, we address the problem of manifold fitting in the ambient space. Inspired by the relation between the eigenvalues of the Laplace-Beltrami operator and the geometry of a manifold, we aim to find a small set of points that preserve the geometry of the underlying manifold. From this relationship, we extend the idea of subsampling to sample points in high-dimensional space and employ the Moving Least Squares (MLS) approach to approximate the underlying manifold. We analyze the two core steps in our proposed method theoretically and also provide the bounds for the MLS approach. Our simulation results and theoretical analysis demonstrate the superiority of our method in estimating the underlying manifold.

\vspace*{.3in}



\section{Introduction}
The digital age in which we live has resulted in huge numbers of high-dimensional sample points. Conducting an analysis of such sample points is extremely difficult and computationally costly because of the curse of dimensionality. Traditional statistics deals with observations that are essentially elements of a real vector space. However, today, many problems of statistical interest in the sciences require the analysis of sample points that consist of more complex objects, taking values in spaces that are naturally not (Euclidean) vector spaces but still feature some geometric structure.  For instance, in object recognition, an image with a resolution of $64\times64$ is directly represented as a vector in a $4,096$-dimensional vector space. For a vector space of such high dimensions, the sample points are obviously too sparse for efficient processing and analysis. More specifically, one million sample points in a 100-dimensional vector space is considered too small a set for analysis. However, with the assumption that the sample points are close to a manifold of much lower dimensions, e.g. a 6-dimensional manifold, then, in theory, one million sample points is enough to extract the relevant information required.

Manifold learning aims to handle the curse of dimensionality by assuming that the sample points are situated near or on a manifold with dimensions much lower than the dimension of the ambient space. To be precise, the high-dimensional sample points are assumed to be locally homeomorphic to the lower-dimensional Euclidean space (the manifold), while the form of the manifold is not necessarily given explicitly. In recent years, manifold learning has been implemented in many types of research, such as pattern recognition and big data analytics. Manifold learning can be split into two different areas of research, the most common being dimension reduction, for example Principal Component Analysis (PCA) \cite{PCA}, Isomap \cite{isomap} and Local Linear Embedding (LLE) \cite{LLE}. Here, the sample points are embedded into a lower-dimensional Euclidean space that can then be used for other types of analysis, such as classification. The key assumption for dimension reduction is that the sample points have a local linear structure, and the mapping constructed preserves some kind of distance. The other area of research involves approximating the underlying manifold directly in the ambient space from sample points. The goal is to project the sample points onto the underlying manifold to denoise the sample points. This area of research is different from dimension reduction in the sense that the projected points still lie in the original ambient space.

It is of great statistical interest to estimate the manifold in ambient space in order to quantify the uncertainty of the learned manifold. For many real-life applications, in particular, it is more convenient to visualize the sample points in the ambient space than in a lower-dimensional space. It should be noted that all the previously mentioned algorithms only aim to find a global embedding of the sample points in a space whose dimensions are much smaller than the input space (mainly focusing on the lower-dimensional geometric properties of the manifold), and therefore are not necessarily able to output a manifold in the ambient space. To the best of our knowledge, less attention has been paid to denoising or estimating the underlying manifold from the sample points in manifold learning. One recent exception is the seminal work of \ctn{fefferman}, who developed an algorithm that outputs a provably smooth manifold from sample points. They performed an exhaustive search to define an approximate squared-distance function from the sample points aiming for the manifold learning. \ctn{KDEandPCA} also demonstrated two different methods of using the approximate squared-distance function to approximate the underlying manifold from sample points. To differentiate, the researchers eliminated the use of the exhaustive search. The study of manifold learning has also been linked to the generalized Whitney problem. \ctn{feffermanwhitney1} aimed to estimate a Riemannian manifold $(\mathcal{M},g)$ to approximate a metric space $(\mathcal{X},d_{\mathcal{X}})$. Another technique using a statistical approach relying on a graph-based diffusion process is discussed in \ctn{noisy1}. On the other hand, for noiseless manifold sample points, some research has been conducted on the reconstruction of the manifold \cite{ManifoldRecon1,ManifoldRecon2,ManifoldRecon3}. For specific types of sample points like noisy images, Locally Linear Denoising (LLD) is a method developed by \ctn{noisy2} to denoise the sample points by exploiting the structure of the underlying manifold. Theoretically, \ctn{minmaxME} have proven a rate of convergence for approximating a manifold given sample points; the optimal rate of convergence depends only on the sample points and the intrinsic dimension of the manifold. Recently, \ctn{yaoxiamf} proposes a new manifold-fitting method. Their output manifold is constructed by directly estimating the tangent space at the projected points on underlying manifold, rather than at the sample points, to decrease the error caused by the noise.

\subsection{Overview and Contribution}

In this paper, we focus on manifold fitting. Given a set of high-dimensional sample points, we assume that the sample points lie near (presumably with noise) a smooth manifold of much lower dimensions. We also assume that the intrinsic dimension of the manifold is known. This paper focuses on fitting a manifold to the sample points and mapping it to the underlying manifold; the main goal is to reduce the computational complexity when we perform manifold fitting. Our method is different from dimension reduction in the sense that the output of our algorithm does not consist of lower-dimensional sample points that are projected from the sample points. Instead, the output is a fitted manifold that lies in the original ambient space. Our algorithm is also able to handle sample points that lie in high-dimensional space; this will be further established in our numerical simulations when we apply our algorithm to image sample points.  Our algorithm consists of two core parts: a subsampling portion and a manifold-fitting approach.

With the subsampling approach, we utilize the spectrum of the Laplace–Beltrami operator. This operator is the extension of the Laplace operator in Euclidean space to the manifold. The eigenvalues of the Laplace–Beltrami operator can be used to identify a manifold up to isometry \cite{shapeDNA}. When sample points are given, the eigenvalues of the Laplace–Beltrami operator can be used to select a subset that still preserves the geometry of the underlying manifold. Our idea is that by controlling an appropriate number of sample points, which is usually very small, we aim to find a subset of points that are essential for recovering the manifold. To achieve this, the key is to make use of the Laplace–Beltrami operator to fine-tune a subset of points such that the eigenvalues do not change by a large margin. We illustrate this idea in Section S5 of the supplementary materials.

By using the eigenvalues of the Laplace–Beltrami operator, the subsampled set still preserves the geometry of the underlying manifold, as demonstrated by the example of the unit sphere in Section S5 of the supplementary materials. Therefore, by working with the subsampled set, we can reduce the complexity of the computations and still retrieve the geometry of the underlying manifold.

Once the subsampled set is determined, the second step is to fit a manifold to the subsampled set. In this paper, we choose to apply a nonlinear Moving Least Squares (MLS) method as a candidate in our framework to fit an underlying manifold. The MLS approach to fitting a manifold has been studied by \ctn{MMLS}, who extended the idea of MLS to manifold fitting and provided some theoretical analysis of their method. They chose to work with a constrained local PCA method to obtain the local coordinate system, which is required with the Least Squares approach. However, we choose to work with a non-constrained local PCA to obtain the required local coordinate system. For our MLS approach, we also provide an error bound with respect to Gaussian noise. The error bounds can also be extended to other types of noise, although these are not explored in this paper.

The benefits of the combination of subsampling and manifold fitting algorithms are twofold. First, the subsampling step selects only a small number of samples to fit the manifold, significantly reducing the time cost.  The sampled points are still sufficient to preserve the underlying structure of the manifold. The experiments in Section \ref{Sect5} show that the manifolds we obtain are quite close to the ground-truth manifolds. Second, sampled points have better distribution properties than the original data points. Therefore, we are able to obtain more accurate results for manifold fitting from the sampled points than directly from the original data. Thus the numerical performance of our method in Section \ref{Sect5} strongly supports the analysis above.

\subsection{Organization}

The rest of the paper is organized as follows. In Section \ref{Sect2}, we discuss the first main step of our algorithm, the subsampling step. In this section, we define the measure of a point and show how it can be used to perform subsampling. The second main step is manifold fitting, which involves MLS and will be presented in Section \ref{Sect3}. We also define MLS and show how it can be used for surface approximation. In Section \ref{sect:errorbounds}, we present theoretical results for the accuracy of tangent-space estimation with sample points, using PCA and the manifold-fitting approach. The theoretical results for the tangent-space estimation are extensions of the work of \ctn{PCA_tangentestimation2}. We also briefly describe our algorithm here. 
In Section \ref{Sect5}, we present our simulation results. 
Finally, Section \ref{Sect6} concludes the paper and suggests possible directions for future research.

\subsection{General Notations}

Throughout this paper, we use sample points to represent a set of distinct random samples of the manifold with noise, unless otherwise stated. The term measure that is used in this paper is to quantify the change in the spectrum of the Laplace–Beltrami operator with regard to a point. We use $X$ to represent a set with a finite number of sample points, and capital and lowercase boldface characters to represent matrices and vectors, respectively. We also use $\bfI$, to represent the identity matrix and $I$ to represent the number of points in a set.


\section{Subsampling}\label{Sect2}

The subsampling strategy is based on a measure of importance, calculated using a function of the heat kernel of the manifold that is related to the spectrum of the Laplace–Beltrami operator. We choose to work with the heat kernel as it allows us to easily quantify the change in the spectrum of the Laplace–Beltrami operator.
We first quantify the importance of a point and formally introduce the measure of this importance. Following this, we propose our subsampling strategy. Finally, we illustrate the effect of the subsampling strategy with a simple example.

\subsection{Measure of a Point}
\ctn{spectral} innovatively defined and used a measure function to determine how much a new point affects the eigenvalues of the heat kernel of a manifold.

\bed{(Measure Function)}
The measure function, $s:\mathcal{M}\rightarrow [0,1]$, is given by
\beq\label{eq:measure}
s(\bfx)=1-\frac{\mathbf{h}_t^T\mathbf{H}_t^{-1}\mathbf{h}_t}{K(\bfx,\bfx,t)},
\eeq
where $\mathbf{h}_t$ is a vector such that $(\mathbf{h}_t)_i=K(\bfx,\bfx_i,t)$ and $\mathbf{H}_t$ is a matrix such that $(\mathbf{H}_t)_{ij}=K(\bfx_i,\bfx_j,t)$.
\eed
Here, $K(\cdot,\cdot,t)$ represents the heat kernel of the manifold. The derivation of equation \eqref{eq:measure} can be found in Appendix B of \ctn{spectral}, and so is not reproduced here. Equation \eqref{eq:measure} shows that the measure of a point $\bfx$ can be calculated using the heat kernel of the manifold. As shown in \ctn{spectral}, the measure function can also be used to determine the perturbation of the eigenvalues of the Laplace–Beltrami operator. In \ctn{spectral}, it is shown that, if a point has a large measure, it could cause a large perturbation on the spectrum. In the algorithm, this means that, for a point $\bfx$ that has a large measure, the point contains important information about the geometry of the underlying manifold, and therefore should be considered a subsampled point.

However, for any abstract manifold with intrinsic dimension $m$, we do not know the exact form of the heat kernel. Thus, a direct computation of the measure function for an abstract manifold is not possible since it depends on the heat kernel. Recall that the heat kernel is used to approximate the Laplace–Beltrami operator when $t\rightarrow 0$. There exists an asymptotic expansion for $K(\bfx_1,\bfx_2,t)$,  and \ctn{heatkernel_approximation} carefully studied the asymptotic behavior of $K(\bfx_1, \bfx_2, t)$.  Remark 2.2 in that author's work may provide more understanding of the asymptotic expansion, and Remark 2.1 indicates that, if $\bfx_1$ and $\bfx_2$ lie in a $\delta_t$-neighborhood of $\bfx$ where $\delta_t=o(t^{1/4})$, when $t\rightarrow 0$, 
\beq\label{eq:kernel_approximation2}
K(\bfx_1,\bfx_2,t)=\frac{1}{(4\pi t)^{m/2}} \exp\Bigg(\frac{-\|E^{-1}_\bfx(\bfx_1)-E^{-1}_\bfx(\bfx_2)\|^2}{2t}\Bigg)(1+o(1)),
\eeq
where $E^{-1}_\bfx$ is the logarithm map, which takes a neighborhood of $\bfx$ in $\mathcal{M}$ and maps it to the tangent space $T_\bfx\mathcal{M}$. 

In accordance with \eqref{eq:kernel_approximation2}, we will define a local measure function that will be used for our algorithm to approximate \eqref{eq:measure}.

\bed{(Local Measure Function)}\label{def:localmeasure}
Let $\mathcal{M}$ be a smooth manifold with an intrinsic dimension $m$, and $N_\mathcal{M}(\bfx,r)$ be the set of neighborhood points on $\mathcal{M}$ that are of distance $r$ away from $\bfx\in\mathcal{M}$. Then, the local measure on $N_\mathcal{M}(\bfx,r)$ is given by
\beq \label{eq:localmeasure}
s_N(\bfx)=1-\frac{\mathbf{\hat{h}}^T\mathbf{\hat{H}}^{-1}\mathbf{\hat{h}}}{\hat{K}(\bfx,\bfx,t_0)},
\eeq
where $\hat{K}$ is the heat kernel given in \eqref{eq:kernel_approximation2} for a fixed time $t_0$, and $\mathbf{\hat{h}}$ and $\mathbf{\hat{H}}$ depend on the heat kernel $\hat{K}$.
\eed
For our numerical calculations, we omit $1+o(1)$ and use 
\beq\label{eq:localheatkernel}
\hat{K}=\frac{1}{(4\pi t)^{m/2}} \exp\Bigg(\frac{-\|E^{-1}_\bfx(\bfx_1)-E^{-1}_\bfx(\bfx_2)\|^2}{4t}\Bigg)
\eeq
for the approximations.

\begin{remark}
Intuitively, this definition for the heat-kernel approximation makes sense. When we consider a small patch of the manifold around the point $\bfx$, and this patch is close to the tangent space of the manifold at $\bfx$, we can then expect the heat kernel of the manifold to behave similarly to the heat kernel in Euclidean space. We use the simple example of the three-dimensional hyperbolic space $\mathbb{H}^3_k$ with $-k^2$ curvature: its heat kernel is
$$K_{\mathbb{H}^3_k}(\bfx,\bfy,t)=\frac{1}{(4\pi t)^{m/2}} \frac{k\|\bfx-\bfy\|}{\sinh(k\|\bfx-\bfy\|)} \exp\Bigg(\frac{-\|\bfx-\bfy\|^2}{4t} - k^2t\Bigg).$$
Here, it is obvious to see that, as $k\rightarrow0$, $K_{\mathbb{H}^3_k}(\bfx,\bfy,t)\rightarrow K(\bfx,\bfy,t)$.
\end{remark} 

\subsection{The Subsampling Algorithm}\label{subs:algorithm1}

Given a set $X$ that lies close to or on a manifold of dimension $m$, we aim to get a subset, $\hat{X}\subset X$, such that $\hat{X}$ contains points with considerable measure. The points in $\hat{X}$ will be used as anchor points to fit a manifold to the sample points. Introducing the subsampling step has two purposes: to reduce the computational complexity of the MLS (as opposed to using the entire set $X$ for the anchor points) and to obtain a better distribution of anchor points, and hence a more consistent fit of the manifold.

To this end, we construct $\hat{X}$ from an iterative algorithm. First, we define an empty set $\hat{X}$, and then iteratively check and add points from $X$ until we have gone through all points. Let $X=\{\bfx_1,\ldots,\bfx_k\}$ and $\hat{X}=\emptyset$. At each step $i$, we remove a random point $\bfx$ and points from $X$ that are of distance $r$ or less from $\bfx$, and call this set $N_i$, as shown in Figure \ref{fig:1a}. We then find $\hat{X}_i\subset \hat{X}$ such that points in $\hat{X}_i$ are also of distance $r$ or less from $\bfx$. For each point in $N_i$, we calculate the local measure defined in Definition \ref{def:localmeasure} using points in $\hat{X}_i$. If this value is more than a threshold $\epsilon$, we add it to $\hat{X}_i$. At the end of the $i$-th iteration, we add the points from $\hat{X}_i$ to $\hat{X}$. This is illustrated in Figure \ref{fig:1b}. We repeat this until $X$ is empty.

\begin{figure}[h!]
\begin{center}
\subfigure[]{
\includegraphics[height=1.4in,width=1.8in]{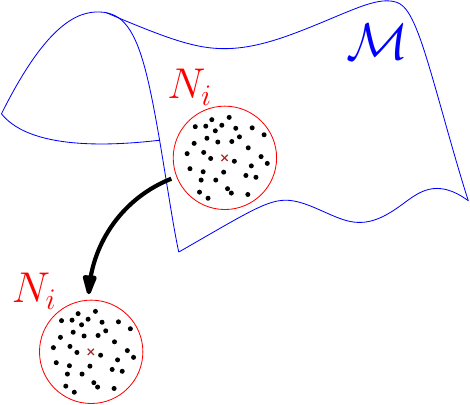}
\label{fig:1a}}
\hspace{.5in}
\subfigure[]{
\includegraphics[height=1.4in,width=1.8in]{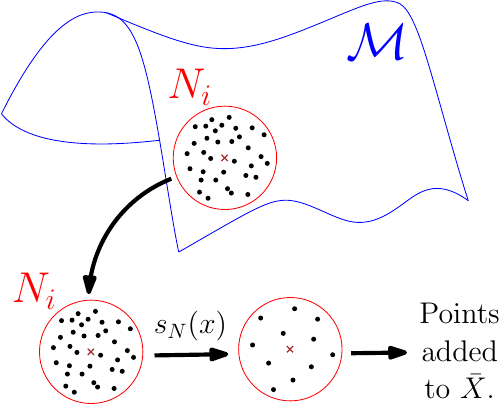}
\label{fig:1b}}
\end{center}
\caption[Illustration of the subsampling algorithm at the $i$-th iteration]{Illustration of the subsampling algorithm at the $i$-th iteration.}
\label{fig:three graphs}
\end{figure}

The details of the subsampling algorithm are in Algorithm \ref{alg1}. Here, $s_{\hat{X}_i}(\bfy_i)$ refers to the local measure function defined in equation (\ref{eq:localmeasure}) of this paper.

\vspace{9pt}\begin{algorithm}[H]\label{alg1}
\algsetup{linenosize=\tiny}
\scriptsize
 \KwData{Initial point set $X$, radius $r$, manifold dimension $d$, threshold $\epsilon$}
 \KwResult{Subsampled point set $\hat{X}$}
 set $\hat{X}=\emptyset$\;
 \While{$X\neq \emptyset$}{
  remove a random point $\bfx_i$ from $X$\;
  find local neighbors $N_i$, with distance $\leq r$ from $\bfx_i$, in $X$\;
  add $\bfx_i$ to $N_i$\;
  find local neighbors $\hat{X}_i$ with distance $\leq r$ from $\bfx_i$ in $\hat{X}$\;
  \While{$N_i \neq \emptyset$}{
  remove a random point $\bfy_i$ from $N_i$\;
  \eIf{$\hat{X}_i= \emptyset$}{
   add $\bfy_i$ to $\hat{X}$ and $\hat{X}_i$\;
   }
   {
  find $s_{\hat{X}_i}(\bfy_i)$\;
  \If{$s_{\hat{X}_i}(\bfy_i)>\epsilon$}{
   add $\bfy_i$ to $\hat{X}$ and $\hat{X}_i$\;
 }
   }
  }
 }
 \caption{Subsampling Algorithm}
\end{algorithm}

In our simulation, the parameters in Algorithm \ref{alg1} can be determined automatically according to the data. The radius $r$ determines the number of neighborhood points. Ideally, we want $r$ to be such that the points on the manifold are related approximately in a linear sense. Empirically, we always set
$
r = p\max_{i,j}\|\bfx_i-\bfx_j\|_2
$
with $p = 0.025$ for simulation. The exact dimensions of the underlying manifold are usually not known. However, these dimensions can be approximated using the methods of \ctn{Estimate_dim_1} and \ctn{Estimate_dim_2}.  The threshold $\epsilon$ is used to control the density of the subsampled set. A large value of $\epsilon$ will translate to a smaller subsampled set,  while a smaller $\epsilon$ value will give a larger subsampled set. In our setting for simulation, we set $\epsilon = 10^{5d-10}$, where $d$ is the estimated dimension of the underlying manifold.

\subsection{The Observation of Subsampling}

Here, we discuss one crucial and interesting observation about the output of the subsampling approach; this discussion requires the definition of an $\epsilon$-separated set.

\bed[$\epsilon$-separated set]
Let $(X,d)$ be a metric space. A subset $A$ is said to be an $\epsilon$-separated set only if, for every $\bfx\neq \bfy\in A$, we have $d(\bfx,\bfy)\geq \epsilon$, where $d$ is the geodesic on the manifold.
\eed

\begin{observation}\label{obser:1}
Given a dense set $X$, the subsampling algorithm will output a subset of points, which is an $\epsilon$-separated set, with $\epsilon$ depending on the threshold of the local measure function.
\end{observation}

With the subsampling algorithm, we pick points that have a measure larger than a certain tolerance. Intuitively, given a subset of accepted points, $X$, we compute the measure of a new point $\bfx$ based on $X$ as follows: if $\bfx$ is close to some point $\bfy\in X$, i.e.  $d(\bfx,\bfy)<\epsilon$, we expect its measure to be small, which means that this point would not significantly perturb the eigenvalues. Thus, we will not add $\bfx$ to $X$. On the other hand, if $\bfx$ is far away from all points $\bfy\in X$, i.e.  $d(\bfx,\bfy)\geq\epsilon$, we expect it to cause a large perturbation of the eigenvalues of the Laplace–Beltrami operator, and we will therefore add $\bfx$ to $X$. We will thus arrive at a subset of the input set $X$ that is an $\epsilon$-separated set. An example is provided below in the form of Figure \ref{fig:randsample}, which shows a random sampling of the unit sphere. The subsampled algorithm is applied with three different starting points, and the results are shown in Figures \ref{fig:subsampledset1}, \ref{fig:subsampledset2}, and \ref{fig:subsampledset3}. It can be observed that the subsampled points are spaced apart, unlike in the original dense sampled set.

\begin{figure}[h!]
\begin{center}
\subfigure[]{
\includegraphics[width=1.22in]{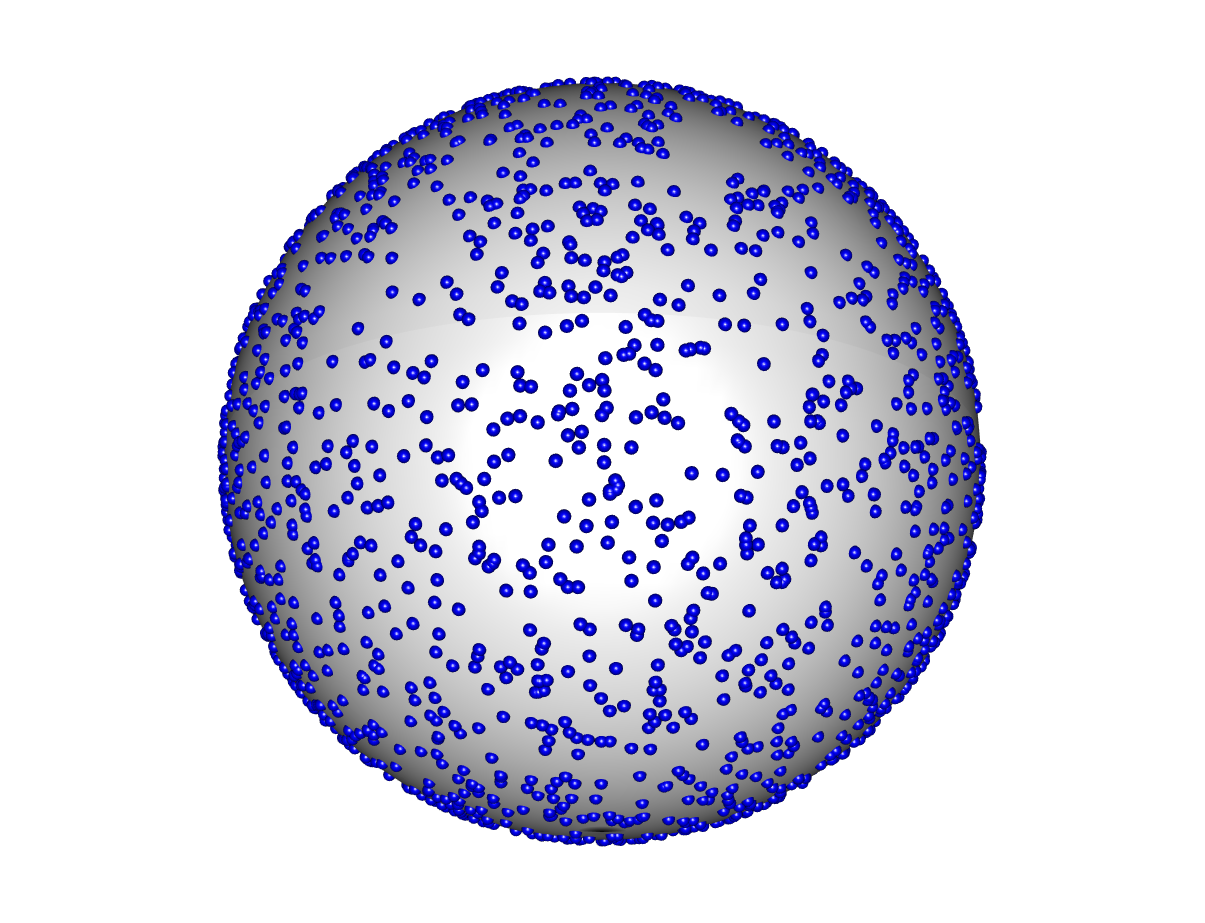}
\label{fig:randsample}}
\subfigure[]{
\includegraphics[width=1.22in]{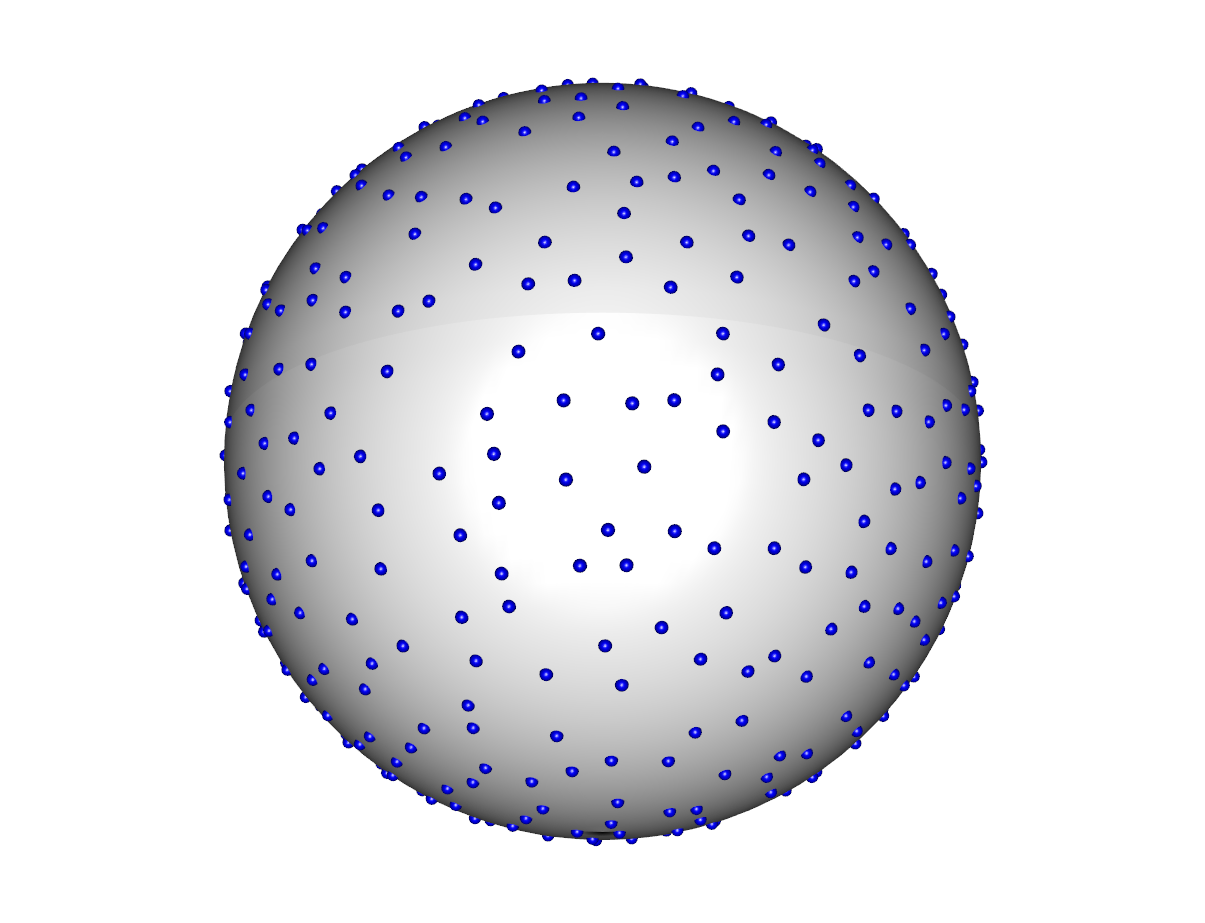}\label{fig:subsampledset1}}
\subfigure[]{
\includegraphics[width=1.22in]{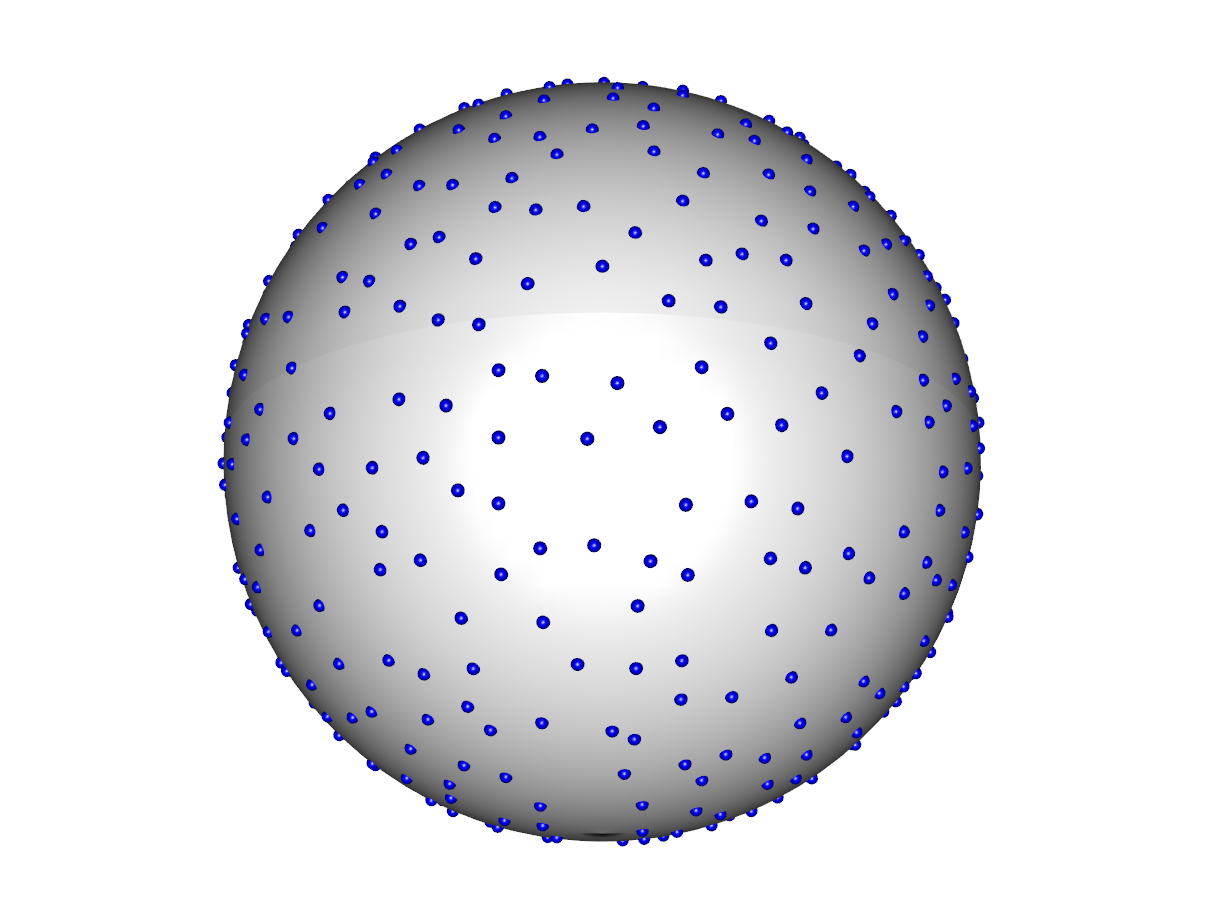}\label{fig:subsampledset2}}
\subfigure[]{
\includegraphics[width=1.22in]{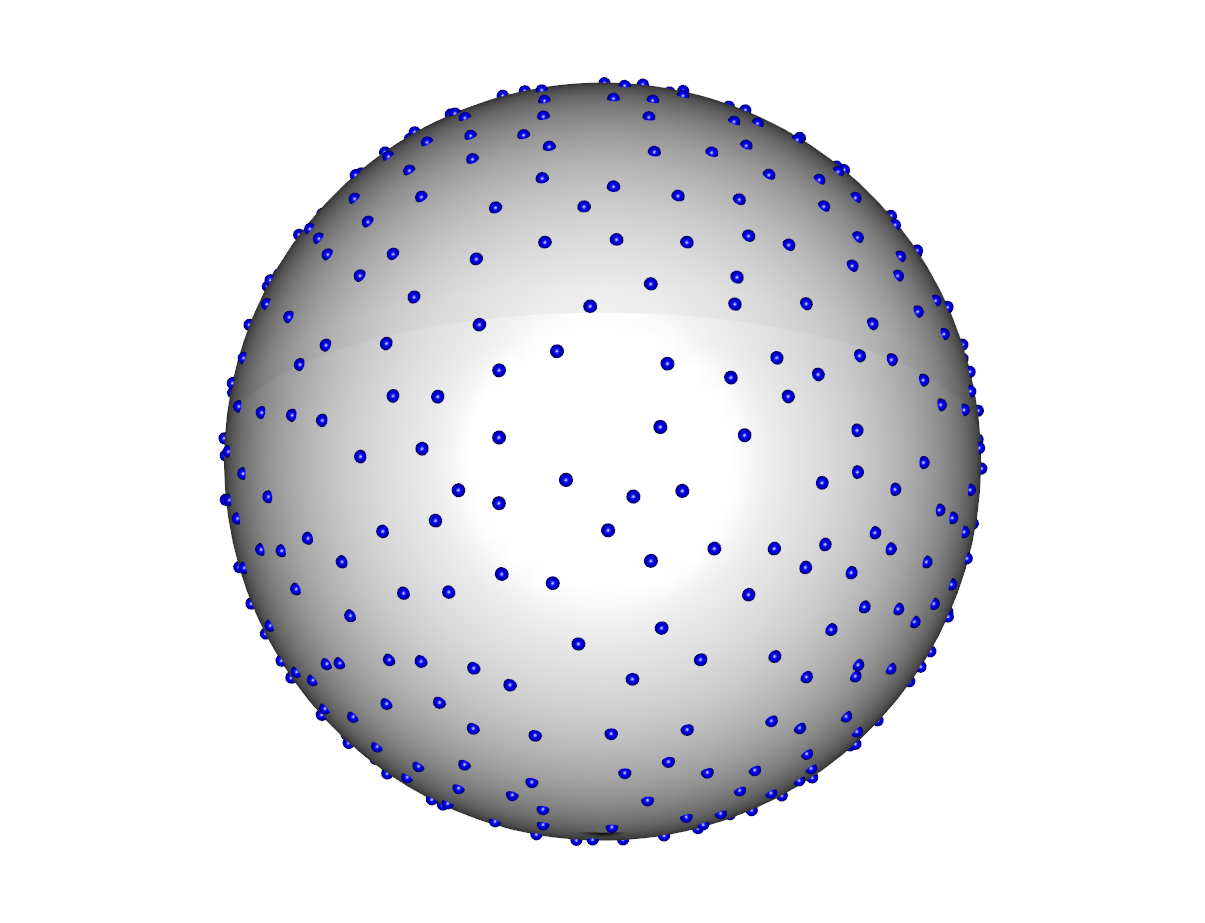}\label{fig:subsampledset3}}
\end{center}
\caption{(a) Random sampling of a sphere. (b)–(d) Subsampled set with different starting point.}
\label{fig:subsphere}
\end{figure}

With a different starting point, the algorithm outputs a different subsampled set, as shown in Figure \ref{fig:subsphere}. However, each subsampled set will still contain the important geometric properties of the manifold. This can be observed in Figure \ref{fig:subsphere}, which shows that the different subsampled sets still sample the entire sphere. The output that this greedy algorithm provides is within the purview of our research, as we aim to reduce the computational complexity by working with a smaller subsampled set for manifold fitting.


\section{Manifold Fitting}\label{Sect3}
In this section, we will discuss the manifold-fitting step, which involves MLS. We will define the MLS for function approximation and discuss how MLS is used in a local setting for surface approximation. 
\subsection{MLS for Function Approximation}
Let $\{\bfx_i\}_{i=1}^I$ be a set of sample points in $\mathbb{R}^m$ and let $\{f(\bfx_i)\}_{i=1}^I$ be the corresponding set of sampled function values at those points for some function $f:\mathbb{R}^m\rightarrow \mathbb{R}$. We want to approximate the function $f$ with a smooth function $\hat f$. The MLS approach was developed to find a function such as $\hat f$. In this approach, $\hat f=\hat f_q$, with $\hat f_q$ being a $q$-degree polynomial function. The $q$-degree polynomial approximation of $f$ at a point $\bfx\in\mathbb{R}^m$ using MLS is defined by $\hat f_q(\bfx)$ such that 
$$\hat f_q(\bfx)=\argmin_{\hat f\in \Pi_q^m} \sum_{i=1}^I (\hat f(\bfx_i)-f(\bfx_i))^2w(\|\bfx-\bfx_i\|),$$
where $w(\cdot)$ is a decreasing non-negative weight function and $\Pi_q^m$ is the space of polynomials of degree $q$ in $\mathbb{R}^m$.

Here we introduce an equivalent definition for the same MLS approximation that can be found in \ctn{approxpower_MLS}. We introduce this equivalent definition as it will enable an easier error analysis in Subsection \ref{sect:errorMLS}.

\ctn{approxpower_MLS} defined MLS in terms of approximating a bounded linear functional, $L$. Within the scope of this paper, the bounded linear functional is defined as $f$ itself, i.e. $L(f)=f(\bfx)$, and we want to find the $\hat L$ that approximates $L$. We use the same notations as in the previous definition of the MLS. The approximation is then defined as
$$\hat L(f)=\sum_{i=1}^I f(\bfx_i)a_i,$$
where $\{a_i\}$ minimizes
$$Q=\sum_{i=1}^I w(\|\bfx-\bfx_i\|)a_i^2,$$
subject to the constraints
$$\sum_{i=1}^I a_ip_j(\bfx_i)=p_j(\bfx), \quad j=1,\cdots,J,$$
where $\{p_j\}_{j=1}^J$ span $\Pi_q^m$.

\subsection{MLS Local Projection}
Let $X$ be in $\mathbb{R}^n$ and assume that $X$ lies close to a manifold of dimensions $m<n$. Let $N_r(\bfx)=\{\bfy\ :\ \bfy\in X\ \text{and}\ \|\bfy-\bfx\|\leq r\} = \{\bfy_i\}_{i=1}^I$, illustrated in Figure \ref{fig:2a} below by the black points. The aim is to find a smooth surface that approximates the points in $N_r(\bfx)$ and map $N_r(\bfx)$ to this surface.

We apply PCA to $N_r(\bfx)$ and use the first $m$ principal components for the local coordinate system shown in Figure \ref{fig:2b}. We map $N_r(\bfx)$ to the $m$-dimensional local coordinate system. Let the mapped points be $N_{r,proj}(\bfx) = \{\bfy_{i,proj}\}_{i=1}^I$, illustrated in blue in Figure \ref{fig:2c}. We then use MLS for function approximation to obtain a polynomial $p(\cdot) : \mathbb{R}^m \rightarrow \mathbb{R}^n$, which approximates the function $f$ where $f(\bfy_{i,proj})=\bfy_i$, as shown in Figure \ref{fig:2d}.

\begin{figure}[h!]
\begin{center}
\subfigure[]{
\includegraphics[height=1.1in,width=1.05in]{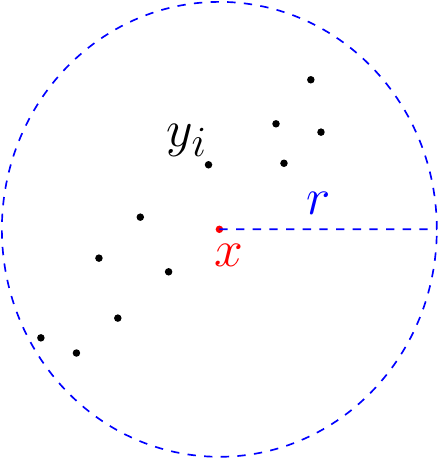}
\label{fig:2a}}
\hspace{0.1in}
\subfigure[]{
\includegraphics[height=1.1in,width=1.05in]{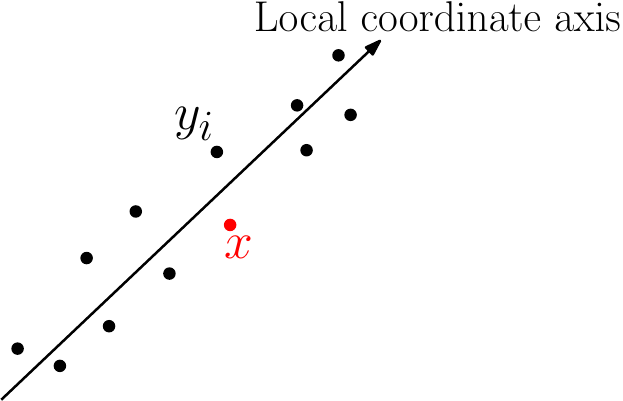}
\label{fig:2b}}
\hspace{0.1in}
\subfigure[]{
\includegraphics[height=1.1in,width=1.05in]{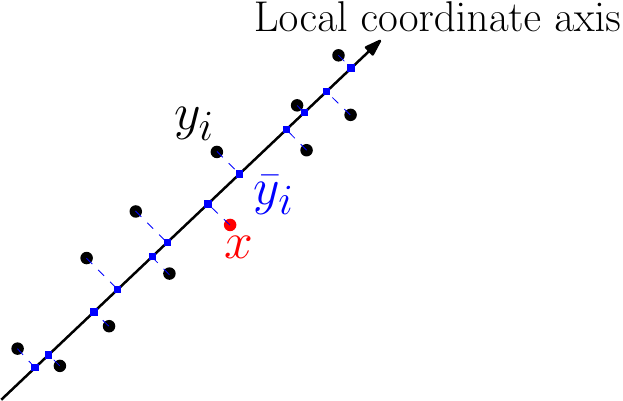}
\label{fig:2c}}
\hspace{0.1in}
\subfigure[]{
\includegraphics[height=1.1in,width=1.05in]{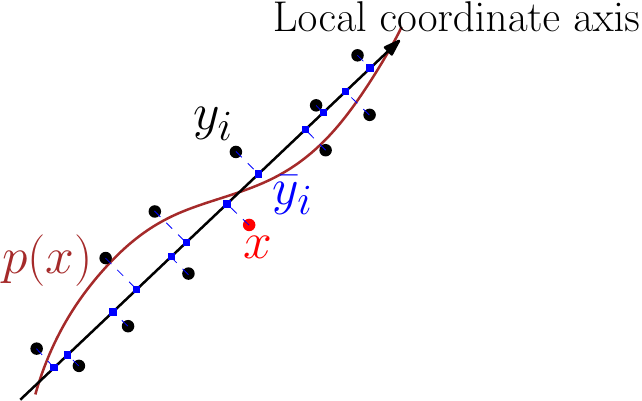}
\label{fig:2d}}
\end{center}
\caption{Illustration of the MLS for local projection.}
\end{figure}

\subsection{The Algorithm of Manifold Fitting}

\vspace{9pt}\begin{algorithm}[H]\label{alg2}
\algsetup{linenosize=\tiny}
\scriptsize
 \KwData{Subsampled point set $S$, number of neighbors $k$, polynomial degree $q$ }
 \KwResult{Projected subsampled set $\hat S$}
 set $\hat S=\emptyset$\;
 \For{each $\bfx$ in $S$}{
 find $k$ nearest neighbors, $N(\bfx)$, of $\bfx$ in $S$\;
 get a local coordinate system for $N(\bfx)$ using PCA\;
 find the representation for $\bfx$ and $N(\bfx)$ using the local coordinates\;
 apply MLS to find a vector polynomial function of degree $m$ that approximates the mapping from the local coordinates to $N(\bfx)$\;
 find $\bfx_{proj}$, which is the projection of the local coordinates of $\bfx$ using the polynomial function\;
 add $\bfx_{proj}$ to $\hat S$;
 }
 \caption{Manifold Fitting of $X$}
\end{algorithm}\vspace{9pt}

\phantom{a}

Here, we provide details of the MLS algorithm; see Algorithm \ref{alg2}. The parameters $k$ and $q$ can be determined from the distribution of the data set. The degree $q$ refers to the degree of the polynomial function used during the MLS algorithm. A higher value of $q$ would give a better MLS estimate. However, by increasing $q$, we would also increase the computational complexity as it would lead to a larger degree of freedom (i.e. a larger number). In our simulation, a small $q$, such as $1$ or $2$, is sufficient to obtain a relatively accurate manifold-fitting result. As for $k$, we suggest to choose it such that the neighborhood would contain about $2$ to $3$ percent of the total number of points in $S$.



\section{Error Bounds for Manifold Fitting}\label{sect:errorbounds}
In this section, we present the non-asymptotic analysis for the manifold-fitting step in Subsection \ref{sect:PCAerror} and Subsection \ref{sect:errorMLS}. The former discusses the error bounds for tangent-space approximation using PCA, and the latter discusses the error bounds for the MLS.

\subsection{Error Bounds for Tangent-Space Approximation}\label{sect:PCAerror}
This subsection provides a non-asymptotic analysis of tangent-space perturbation in accordance with the main result from \ctn{PCA_tangentestimation2}. The analysis of the perturbation is based on the Frobenius norm on the difference between the projection matrices. From this bound, we are able to derive an upper bound for the principal curvature and the noise level such that the analysis is well defined. We are also able to derive a range for the neighborhood radius in the tangent space. 

We first set up the framework needed for the analysis and state the main results for the perturbation of the tangent space approximated with PCA. Following this, we analyze the error bound and give the conditions under which the error bound makes sense. The conditions are given in terms of the curvature, noise, and neighborhood radius.

Consider $N_{\epsilon}(\bfy)$ to be the set of points that is of distance  $\epsilon$ away from $\bfy$. For points lying directly on the manifold in $N_{\epsilon}(\bfy)$, we can uniquely represent them as 
$$[\bfx, f_1(\bfx), \ldots, f_{n-m}(\bfx)]^T.$$
Here, $\bfx=[x^{(1)},\ldots, x^{(m)}]^T$ denotes the coordinates of the projection of the point onto $T_\bfy\mathcal{M}$. In the above representation, the point $\bfy$ is taken to be the reference point. We further assume that the embedding $f_l,\ l=1,\ldots,n-m$ is $\mathcal{C}^a, a>2$ and define $\mathcal{K}_{l,1},\ldots,\mathcal{K}_{l,m}$ as the principal curvatures of the hypersurface
$$\mathcal{S}_l=\{[\bfx\ f_l(\bfx)]^T : \bfx\in T_y\mathcal{M}\} \in \mathbb{R}^{m+1}, \quad l=1,\cdots,n-m.$$
With second-order Taylor expansion, we obtain
$$f_l(\bfx)=\frac{1}{2}(\mathcal{K}_{l,1}(x^{(1)})^2 + \cdots + \mathcal{K}_{l,m}(x^{(m)})^2) + O(\|\bfx\|_2^3), \quad l=1,\ldots, n-m.$$
We truncate the Taylor expansion and consider a quadratic approximation,
\beq\label{eq:quad_embed}
f_l(\bfx)=\frac{1}{2}(\mathcal{K}_{l,1}(x^{(1)})^2 + \cdots + \mathcal{K}_{l,m}(x^{(m)})^2), \quad l=1,\ldots, n-m,
\eeq
to describe the manifold locally.
\begin{remark}
In our analysis, we truncate the Taylor expansion and consider the manifold to be an exact quadratic embedding of the form \eqref{eq:quad_embed}. While the manifold in real-life problems is not an exact quadratic embedding, for our analysis we assume that considering a quadratic embedding is sufficiently general. Furthermore, we note that the analysis in \ctn{PCA_tangentestimation} for the noiseless case shows that there is no loss in accuracy for a general case between a smooth embedding and a quadratic embedding.
\end{remark}
We assume that there are $I$ noiseless points,
$$\{\bfy_i\}_{i=1}^I=\Big\{ [x_i^{(1)}, x_i^{(2)}, \ldots, x_i^{(m)}, f_1(\bfx_i), \ldots, f_{n-m}(\bfx_i) ]^T\Big\}_{i=1}^I,$$
in $N_{\epsilon}(\bfy)$. Here, $\bfx_i=[x_i^{(1)}, x_i^{(2)}, \ldots, x_i^{(m)}]^T$ represents the coordinates of the projection of $\bfy_i$ on $T_\bfy S$ for $i=1,\ldots, I$. We define the matrices $\bfX_{noiseless}, \bfE, \bfL$, and $\bfS$ where the $i$-th columns are
\begin{gather*}
(\bfX_{noiseless})_i=\bfy_i, \qquad
\bfE_i=[e_i^{(1)},\cdots,e_i^{(n)}]^T,\\
\bfL_i=[x_i^{(1)}, x_i^{(2)}, \ldots, x_i^{(m)},0, \ldots, 0 ]^T,\qquad
\bfS_i=[0, \ldots, 0, f_1(\bfx_i), \ldots, f_{n-m}(\bfx_i) ]^T,
\end{gather*}
for $i=1,\cdots,I$. The noise matrix is represented by $\bfE$ and we assume that the noise, $\bfE_i, i=1,\cdots I$, for each point, $\bfy_i$, is i.i.d. and follows a multivariate Gaussian distribution. For simplicity, we consider the Gaussian distribution to have zero mean and covariance $\sigma^2\bfI$, i.e. $\bfE_i \sim \mathcal{N}(\bold{0},\sigma^2\bfI), i=1,\cdots,I.$ 

The sample points to be considered are then $\bfX=\bfX_{noiseless}+\bfE=\bfL+\bfS+\bfE$. For any matrix $\bfZ$, let $\bfC_\bfZ$ represent the centered version of $\bfZ$ with respect to the columns. Then, the covariance matrix is given by

\begin{align*}
\frac{1}{I}\mbfX\mbfX^T&=\frac{1}{I}(\mbfL+\mbfS+\mbfE)(\mbfL+\mbfS+\mbfE)^T\\
&=\frac{1}{I}\mbfL\mbfL^T+\Theta,
\end{align*}
where 
$$\Theta=\frac{1}{I}\Big(\mbfS\mbfS^T + \mbfE\mbfE^T +\mbfL\mbfS^T + \mbfS\mbfL^T + \mbfL\mbfE^T + \mbfE\mbfL^T + \mbfS\mbfE^T + \mbfE\mbfS^T\Big) $$
is the perturbation term when we want to retrieve the tangent space at $\bfy$. Working directly with $\mbfL\mbfL^T$  will allow us to retrieve the true tangent space at the point $\bfy$. However, we are unable to do so because of the perturbation term, and we have only $\mbfX\mbfX^T$ to work with.

We will analyze the error caused by the perturbation term in terms of the Frobenius norm of the difference between the orthogonal projectors. Let
$\bfP$ represent the orthogonal projector for the true tangent space associated with $\mbfL\mbfL^T$ and $\hP$ be the projector for the approximated one from $\mbfX\mbfX^T$. Let $r$ represent the radius of the neighborhood on the tangent plane, $T_\bfy\mathcal{M}$, that contains the $I$ points in $N_{\epsilon}(\bfy)$ projected onto $T_\bfy\mathcal{M}$. We define constants $\mathcal{K}$ and $\mathcal{K}^{(+)}$ with respect to the principal curvatures $\mathcal{K}_{i,j}, i=1,\cdots,n-m$ and $j=1,\cdots,m$.
\begin{gather*}
\mathcal{K}=\Bigg[ \sum_{i=1}^{n-m}\sum_{j=1}^{n-m} \Bigg((m+1)\sum_{k=1}^m\mathcal{K}_{i,k}\mathcal{K}_{j,k} - \sum_{\substack{l_1,l_2=1\\ l_1\neq l_2}}^m \mathcal{K}_{i,l_1}\mathcal{K}_{j,l_2}  \Bigg)^2\ \Bigg]^{\frac{1}{2}}\quad \text{and}\\
\mathcal{K}^{(+)} =\Bigg[ \sum_{i=1}^{n-m} \Bigg( \sum_{j=1}^m | \mathcal{K}_{i,j} |^2 \Bigg)^2\ \Bigg]^{\frac{1}{2}}.
\end{gather*}
The constant, $\mathcal{K}$, can be seen as the expectation of the norm of curvature covariance, and $\mathcal{K}^{(+)}$ is a positive term that is used when positive curvature terms are required.
From the main theorem of \ctn{PCA_tangentestimation2}, we have
\beq\label{eq:tangentapprox_main}
Pr\Bigg( \|\bfP-\hP\|_F\leq\frac{2\sqrt{2}\beta}{\delta} \Bigg) \geq1-9e^{-\xi_1^2}-2me^{-\xi_2^2},
\eeq
where $$\beta=\frac{1}{\sqrt{I}}\Bigg[\mathcal{K}^{(+)}r^3\upsilon(\xi_1) + \sigma\sqrt{m(n-m)}\eta(\xi_1,\xi_2) + \frac{1}{\sqrt{I}} \zeta_{\text{numer}}(\xi_1) \Bigg],$$
$$\delta=\frac{r^2}{m+2} - \frac{\mathcal{K}r^4}{2(m+2)^2(m+4)}-\sigma^2\big(\sqrt{m}+\sqrt{n-m}\big)-\frac{1}{\sqrt{I}}\zeta_{\text{denom}}(\xi_1,\xi_2)$$
and $\xi_1$ and $\xi_2$ are probability constants that depend on the dimension of the ambient space and manifold, and the number of points, $I$.

\begin{remark}
The functions, $\upsilon(\xi_1),\eta(\xi_1,\xi_2),\zeta_{\text{numer}}(\xi_1)$, and  $\zeta_{\text{denom}}(\xi_1,\xi_2)$, that appear in the terms $\beta$ and $\delta$ are complicated functions that depend on the probability constants $\xi_1$ and $\xi_2$. They are not presented in this paper, but \ctn{PCA_tangentestimation2} may be referred to for their exact forms.
\end{remark}
From the upper bound given in \eqref{eq:tangentapprox_main}, we are able to derive an informal error bound,
\beq\label{eq:PCA_bound}
\|\bfP-\hP\|_F
\leq 
\frac{\frac{2\sqrt{2}}{\sqrt{I}}\Big[\mathcal{K}^{(+)}r^3 + \sigma\sqrt{m(n-m)}\Big( \sigma + \frac{r}{\sqrt{m-2}} + \frac{\mathcal{K}^{1/2}r^2}{(m+2)\sqrt{2(m+4)}}\Big)\Big]}
{\frac{r^2}{m+2} - \frac{\mathcal{K}r^4}{2(m+2)^2(m+4)}-\sigma^2\big(\sqrt{m}+\sqrt{n-m}\big)},
\eeq
that is sufficient for an analysis of how the curvature, noise levels, and radius affect the tangent-space approximation. The informal bound is derived by momentarily neglecting the probabilistic constants. Intuitively, we should obtain a better tangent-space approximation when we have more sample points. This is demonstrated in \eqref{eq:PCA_bound}, where having more points implies that we have a larger $I$, which would make a smaller upper bound. Furthermore, it can be observed that, for \eqref{eq:PCA_bound} to be well defined, the upper bound on the right-hand side is required to be positive. This gives us the necessary condition in which
\beq\label{EQ:BOUND_CONDITION1}
\mathcal{K}\sigma^2<\frac{m+4}{2\big(\sqrt{m}+\sqrt{n-m}\big)},
\eeq
and
\begin{align}
r^2\in&\Bigg(\frac{(m+2)\big[(m+4)-\sqrt{(m+4)^2-2\mathcal{K}\sigma^2(m+4)(\sqrt{m}+\sqrt{n-m})}\ \big]}{\mathcal{K}},\nonumber\\
&\qquad\frac{(m+2)\big[(m+4)+\sqrt{(m+4)^2-2\mathcal{K}\sigma^2(m+4)(\sqrt{m}+\sqrt{n-m})}\ \big]}{\mathcal{K}}\Bigg).\label{EQ:BOUND_CONDITION2}
\end{align}
The derivations of conditions \eqref{EQ:BOUND_CONDITION1} and \eqref{EQ:BOUND_CONDITION2} are presented in Section S2 of the supplementary materials.

Condition \eqref{EQ:BOUND_CONDITION1} gives an upper bound for the curvature and variance of the noise. From condition \eqref{EQ:BOUND_CONDITION1}, the term $\mathcal{K}\sigma^2$ is required to be smaller than a term that depends on the dimension of the ambient space and manifold. Intuitively, this makes sense as, for any point on the manifold, if the curvature at that point or the variance of the noise in the neighborhood is large, we expect to be unable to obtain an accurate approximation of the tangent space. If the curvature at a point is high, locally a smaller radius is required so that the points behave in a close-to-linear sense. Furthermore, if the variance of the noise is large, the noise might play a larger role in the tangent-space approximation, and we might not be able to obtain an accurate approximation.

Condition \eqref{EQ:BOUND_CONDITION2} gives a range for $r$ such that the right-hand side of \eqref{eq:PCA_bound} is positive. The condition shows that $r$ cannot be a large value as it will cause the error bound to be negative and thus not well defined. Therefore, with a large $r$, we may not be able to obtain an accurate approximation of the tangent space. Condition \eqref{EQ:BOUND_CONDITION2} is somewhat related to condition \eqref{EQ:BOUND_CONDITION1} in the sense that, if $r$ is large, the curvature at that point will have a larger effect on the tangent-space approximation, thereby making the approximation inaccurate, as described above. On the other hand, $r$ cannot be made too small as this will also make the denominator of \eqref{eq:PCA_bound} small, which then gives a large error bound. Thus, the user has to choose $r$ appropriately; a recommendation is provided in Subsection \ref{subs:algorithm1}. We also present a numerical simulation in Subsection S4.3 in the supplementary materials to determine the effect of the radius on the accuracy of the output.

\subsection{Error Bounds for MLS}\label{sect:errorMLS}

This subsection gives the pointwise error analysis for the MLS approximation for sample points. We follow the modality of \ctn{MLS_errbound} to analyze the MLS approximation error, i.e. $|f(\bfx)-\hat f_q(\bfx)|$. \ctn{MLS_errbound} analyze the error bounds based on the derivatives, while we are interested in the pointwise error of the functions. Thus, we follow the proofs closely and work out the pointwise error for $|f(\bfx)- \hat f_q(\bfx)|$ in a similar fashion.

We first set up the framework required and then give the pointwise error bound. We use the multi-index notation $\alpha=(\alpha_1,\cdots,\alpha_m)$, $\alpha!=\alpha_1!\cdots\alpha_m!$, $|\alpha|=\alpha_1+\cdots+\alpha_m$, and $\bfx=[x^{(1)},\cdots,x^{(m)}]^T$, $\bfx^\alpha = (x^{(1)})^{\alpha_1}\cdots(x^{(m)})^{\alpha_m}$. We then define the derivative $D^{\alpha}:=\partial^{\alpha_1}_{x^{(1)}} \cdot\ldots \partial^{\alpha_m}_{x^{(m)}}$. 

Recall that in MLS we approximate the function $f(\bfx)$ with a $q$-degree polynomial function, $\hat f_q(\bfx)$. We let $\{p_{\alpha}\}_{|\alpha|\leq m}$ be the standard basis for $\Pi_q^m$ shifted to $\bfx$, i.e. $p(\cdot)=(\cdot - \bfx)^\alpha$, where $\bfx$ refers to the reference point for MLS approximation. We let $\bfV$ be the generalized Vandermonde matrix, i.e. $\bfV_{i,\beta}=p_{\beta}(\bfx_i)$.

\begin{theorem}\label{THM:MLS_POINTERR}
Let $\hat f_q$ be the approximated $q$-degree polynomial for the function $f^*$ defined by the points $\{(\bfx_i,f^*(\bfx_i))\}_{i=1}^I, (\bfx_i,f^*(\bfx_i))\in \mathbb{R}^m\times \mathbb{R}$, and $f^*(\bfx_i)=f(\bfx_i)+e_i$, where $e_i \sim N(0,\sigma^2)$ and $f$ is the true function. Assume that $f\in \mathcal{C}^{m+1}(\mathbb{R}^m)$ and $D^{\alpha}f(\bfx) < C_{\alpha}$ for $|\alpha|\leq m+1$. Then, for any $k$, with probability at least $1- \frac{2}{\sqrt{2\pi}k}\exp\{-\frac{k^2}{2}\}$,
\begin{align*}
|\hat f_q(\bfx)-f(\bfx)| &\leq O(r^{m+1}) + |Ck\sigma\sqrt{I}|,
\end{align*}
where $C=\max_{i=1}^I \frac{\det ((\bfV^T\bfW\bfV)_{1\leftarrow  (\bfV^T\bfW)_i})}{\det (\bfV^T\bfW\bfV)}$ and $(\bfV^T\bfW\bfV)_{1\leftarrow  (\bfV^T\bfW)_i}$ represents the matrix $\bfV^T\bfW\bfV$, with the first column being replaced by the $i$-th column of $\bfV^T\bfW$, and $\bfW=diag\big(w(\|\bfx-\bfx_1\|),\cdots,w(\|\bfx-\bfx_I)\big)$.
\end{theorem}
The proof can be found in Section S3 of the supplementary materials.
\begin{remark}
We assume that the data set $X = \{\bfx_1,... , \bfx_I\}$ is normalized, that is, $\max_{1\leq i,j\leq I} \|\bfx_i-\bfx_j\| = 1$. Following our recommendation of the parameter setup in Section \ref{Sect2}, we will have $r = 0.025$, which means that the term $O(r^{m+1})$ can be ignored for a large dimension $m$. In the second term $|Ck\sigma\sqrt{I}|$, $C$ and $\sigma$ are intrinsic to the data, and the size of the upper bound is actually controlled by $I$. Fortunately, our subsampling step reduces $I$ while maintaining the underlying structure of the manifold, thus providing a smaller upper bound for the manifold fitting. It is also worth noting that the upper bound of the fitting error provided in Theorem 1 is for illustration purposes only. In practice, the MLS algorithm with a subsampling step provides fairly accurate results for the manifold fitting; see the ablation experiment in Subsection \ref{subsect:abl}. 
\end{remark}


\section{Numerical Results}\label{Sect5}

This section present some numerical results from our method with various low-dimensional manifolds and a real-world data set. The low-dimensional manifolds contained a one-dimensional helix, a six-folded curve, and a two-dimensional Swiss roll, all of which lay in three-dimensional space. We first demonstrate the benefits of the subsampling step through the ablation experiments in Subsection \ref{subsect:abl}. In Subsection \ref{sect:numercal1d}, we compare the results from our algorithm with those from other methods, including local PCA, Isomap, and LLD.  Finally, we adapt our algorithm for the denoising of facial images.

\subsection{Ablation Experiments for Subsampling}\label{subsect:abl}
We execute the MLS method on three low-dimensional manifolds with and without the subsampling step. The results of the manifold fitting are represented in three formats: the mean square error (MSE), the smoothness of the reconstructed manifold, and the CPU time. 

Table \ref{tab:ablation} shows the MSE and CPU time of the MLS algorithm with and without subsampling. For the MLS with the subsampling step, we counted the total time for both the subsampling step and the MLS step. As the table indicates, with the subsampling step performed, the time required for the manifold fitting was significantly reduced for all the three low-dimensional manifolds, which was consistent with our expectation. Meanwhile, subsampling helped the MLS algorithm to obtain a manifold with lower MSE. This was mainly because subsampling allows for a more evenly distributed set of samples.

\begin{table}[htbp]
  \centering
   \resizebox{0.6\textwidth}{!}{ 
    \begin{tabular}{c|cc|cc}
    \hline\hline
    \multirow{2}[4]{*}{Manifold} & \multicolumn{2}{c|}{MSE} & \multicolumn{2}{c}{time (s)} \\
\cline{2-5}   & S & NS & S & NS \\
    \hline
    helix & $4.0\times 10^{-3}$ & $5.4\times 10^{-3}$ & 23.4  & 72.3 \\
    six-folded & $4.7 \times 10^{-5}$ & $6.0 \times 10^{-5}$ & 14.2  & 37.8 \\
    Swiss roll & $2.8 \times 10^{-2}$ & $3.1 \times 10^{-2}$ & 14.9  & 52.7 \\
    \hline\hline
    \end{tabular}}%
    \caption{The result of the ablation experiments for subsampling (S) and non-subsampling (NS).}
  \label{tab:ablation}%
\end{table}%

\begin{figure}[h]
    \includegraphics[width = 1\linewidth, height = 0.28\linewidth]{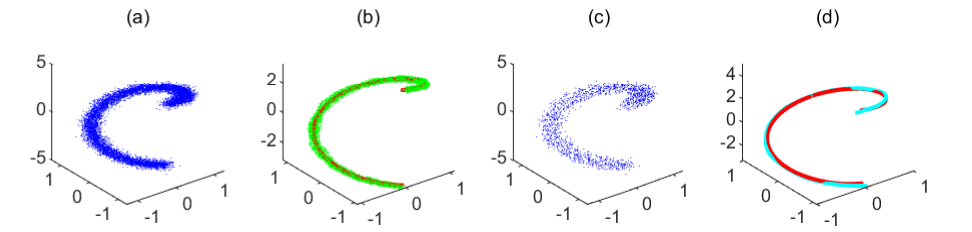}        
    \caption{Ablation experiments for subsampling on helix. (a) Sample points with noise. (b) The result from the MLS algorithm without subsampling (green) and the noiseless sample points (red). (c) The anchor points obtained by subsampling. (d) The result from the MLS algorithm with subsampling (teal) and the noiseless sample points (red).}
    \label{fig:ablationhelix}
\end{figure}

In addition to the observations above, we are also concerned with the smoothness of the obtained manifolds. The helix results are shown in Figure \ref{fig:ablationhelix}, and the results from the other two manifolds may be found in the supplementary materials. The MLS algorithm without the subsampling step yields more concentrated points than the original data points but does not yield a smooth curve. In contrast, the MLS algorithm with the subsampling step yields a smooth curve that is quite close to the noiseless curve. These phenomena suggest that the subsampling step improves not only the efficiency but also the performance of the MLS method.

As the results above indicate, subsampling can be very helpful in improving the performance of manifold fitting, because of the ability of subsampling to obtain $\epsilon$-separated data sets. Although rigorous proof is not available, we still demonstrate this with numerical experiments.
Figure \ref{fig:subsampling} shows the histogram of the distance between each data point and its nearest point. After subsampling, these distances have a larger lower bound, which validates Observation \ref{obser:1} in Section \ref{Sect2}.
\begin{figure}[h]
    \includegraphics[width = 1\linewidth, height = 0.25\linewidth]{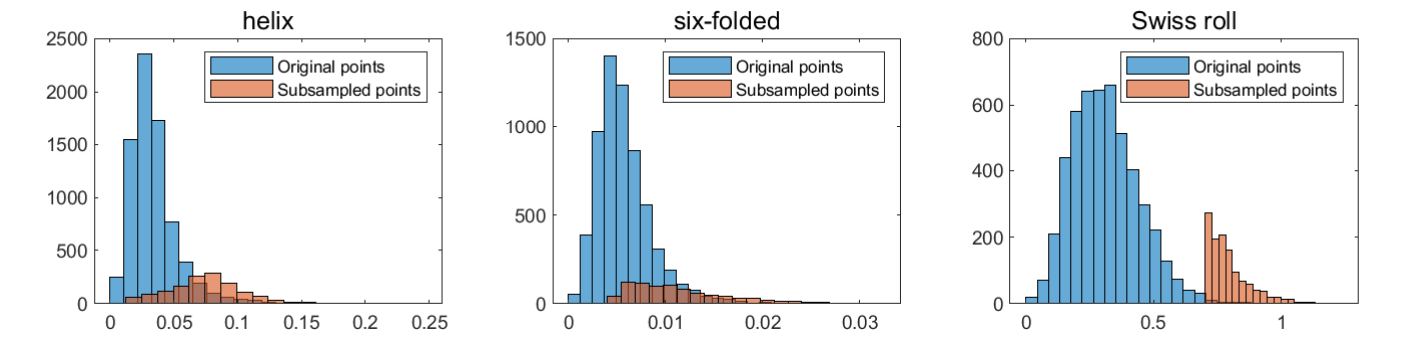}     
    \caption{Histogram of the distance between each point and its nearest point.}
    \label{fig:subsampling}
\end{figure}

\subsection{Comparison of low-dimensional manifold-fitting approaches}\label{sect:numercal1d}
In this subsection, we compare our algorithm with other methods (local PCA, Isomap, and LLD) on three low-dimensional manifolds (helix, six-folded, and Swiss roll).

\begin{figure}[h]
    \includegraphics[width = 1\linewidth, height = 0.45\linewidth]{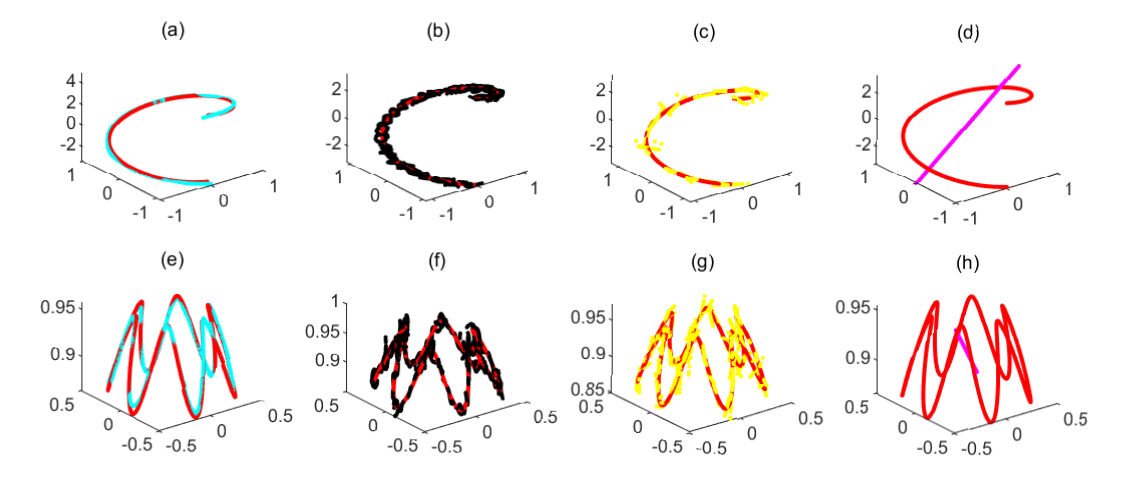}        
    \caption{Simulation of one-dimensional manifold. (a) and (e): The results
of our algorithm (teal) and the noiseless sample points (red). (b) and (f): The
results of LLD (black) and the noiseless sample points (red). (c) and (g): The
results of local PCA (yellow) and the noiseless sample points (red). (d) and (h): The
results of Isomap (pink) and the noiseless sample points (red).}
    \label{fig:com1dim}
\end{figure}

For one-dimensional manifolds, we applied all compared algorithms, and the results are shown in Figure \ref{fig:com1dim}. Our method is able to output smooth curves that are close to the underlying manifolds, as shown in (a) and (e). LLD generates locally smooth irregular curves that deviate significantly from the underlying manifolds. Local PCA is able to approximate the manifold locally with linear vectors and retrieve some geometric properties of the manifold. However, the variance of the sample points can affect the results and cause the local PCA to fail, as shown in (c) and (g). Isomap is a dimension-reduction method, and can be used to visualize the sample points in the lower-dimensional space. However, as shown in the simulations, it cannot be used to retrieve the geometric properties of the manifold in the ambient space.



\begin{figure}[h]
    \includegraphics[width = 1\linewidth, height = 0.25\linewidth]{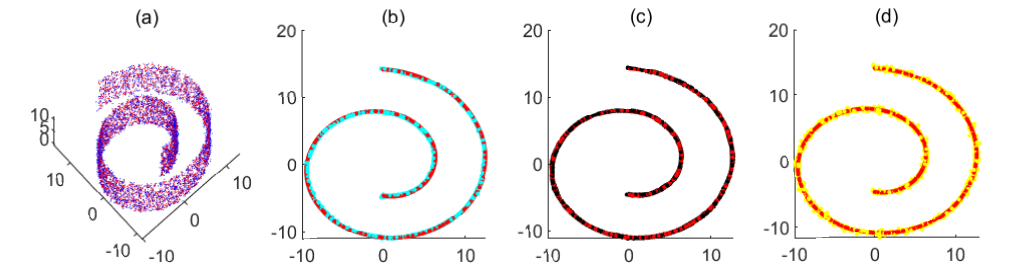}        
    \caption{Simulation of Swiss roll sample points. (a) 3D scatter plot of the
noisy (blue) and noiseless sample points (red). (b) Projected 2D
plot of our method (pink) and the noiseless sample points (red). (c) Projected 2D
plot of LLD method (black) and the noiseless sample points (red). (d) Projected 2D
plot of local PCA method (yellow) and the noiseless sample points (red).}
    \label{fig:com2dim}
\end{figure}
 
For the Swiss roll data, we applied our algorithm to the sample points, and the results are shown in Figure \ref{fig:com2dim}. 
The results from our algorithm, LLD, and local PCA were projected onto the $x$-$y$ plane by removing the $z$ component, and are plotted in (b), (c), and (d) of Figure \ref{fig:com2dim}. Here, it can be observed that our algorithm was also able to denoise the sample points, and the results show that we were able to retrieve the geometric properties of the manifold. Surprisingly, LLD and local PCA were also able to denoise the sample points, unlike in the one-dimensional case.

\begin{remark}
For this simulation, we do not show the results from Isomap, because the sample points would just have been projected onto a two-dimensional plane.
\end{remark}

\subsection{Image Data Simulation}\label{sect:numerical_image1}

The image sample points were taken from the video ``Sub 4'' from \ctn{imageref}. 
The video is a recording of the face of a person in accordance with a set of instructions provided in Table 1 in \ctn{imageref}. The video is accessible through the following link: \url{https://sites.google.com/site/nirdatabase/download}. We chose to work with ``Sub 4'' as we wanted more facial features to test our algorithm.

The video consisted of 16,708 frames, and each frame was cropped and converted to a grayscale image. In addition, we resized each image to $70\times 50$ pixels. Our sample points were then $X=\{\mathbf{X}_1,\ldots,\mathbf{X}_{16708}\}$, where $\mathbf{X}_i$ was a matrix of dimension $70\times 50$. To compare the effectiveness of our method with other methods, we added Gaussian noise with different standard deviations of $0.01, 0.025, 0.05$, and $0.075$ (corresponding to noise levels $1$, $2$, $3$, and $4$, respectively) to the sample points, and ran our algorithm at the four levels of noise. 

To illustrate the comparison with other methods, we selected five different head postures. The full results are available at \url{https://zhigang-yao.github.io/research.html}. The original images of the five different head postures are shown in Figure \ref{fig:original}.

\begin{figure}[h!]
\centering
\resizebox{5.5cm}{1.13cm}{
\includegraphics{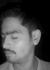}
\hspace{.01in}
\includegraphics{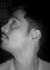}
\hspace{.01in}
\includegraphics{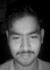}
\hspace{.01in}
\includegraphics{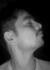}
\hspace{.01in}
\includegraphics{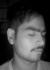}}
\caption{Noiseless images.}
\label{fig:original}
\end{figure}

The results are presented in Figures \ref{fig:image0p01}, \ref{fig:image0p025}, \ref{fig:image0p05}, and \ref{fig:image0p075}, with each figure showing the results for one of the four levels of noise, in ascending order. In each figure, the six rows, from top to bottom, display the results from the noisy image, Isomap, global PCA, local PCA, LLD, and our algorithm.

For each noise level, we used 10 leading eigenvectors for PCA, as the PCA results would not have differed much even if we had worked with more than 10 leading eigenvectors; in addition, the results would actually have worsened if more than 20 principal components had been chosen. For a fair comparison, we also used 10 leading eigenvectors for all the algorithms. To be consistent in our analysis, we chose 150 nearest neighbors for the tangent-space approximation in MLS and used functions of degree one (linear) when applying the MLS step. For every noise level, our algorithm was iterated five times.

\begin{figure}[h!]
\begin{minipage}{0.48\textwidth}
\centering
\resizebox{5.5cm}{1.13cm}{
\includegraphics{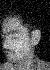}
\hspace{.01in}
\includegraphics{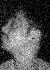}
\hspace{.01in}
\includegraphics{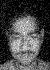}
\hspace{.01in}
\includegraphics{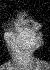}
\hspace{.01in}
\includegraphics{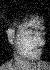}}

\vspace{.05in}

\resizebox{5.5cm}{1.13cm}{
\includegraphics{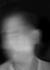}
\hspace{.01in}
\includegraphics{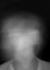}
\hspace{.01in}
\includegraphics{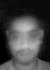}
\hspace{.01in}
\includegraphics{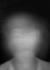}
\hspace{.01in}
\includegraphics{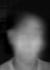}}

\vspace{.05in}

\resizebox{5.5cm}{1.13cm}{
\includegraphics{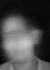}
\hspace{.01in}
\includegraphics{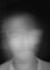}
\hspace{.01in}
\includegraphics{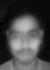}
\hspace{.01in}
\includegraphics{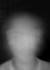}
\hspace{.01in}
\includegraphics{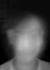}}

\vspace{.05in}

\resizebox{5.5cm}{1.13cm}{
\includegraphics{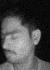}
\hspace{.01in}
\includegraphics{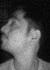}
\hspace{.01in}
\includegraphics{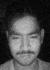}
\hspace{.01in}
\includegraphics{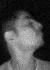}
\hspace{.01in}
\includegraphics{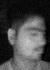}}

\vspace{.05in}

\resizebox{5.5cm}{1.13cm}{
\includegraphics{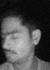}
\hspace{.01in}
\includegraphics{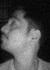}
\hspace{.01in}
\includegraphics{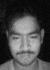}
\hspace{.01in}
\includegraphics{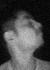}
\hspace{.01in}
\includegraphics{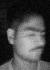}}

\vspace{.05in}

\resizebox{5.5cm}{1.13cm}{
\includegraphics{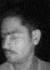}
\hspace{.01in}
\includegraphics{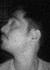}
\hspace{.01in}
\includegraphics{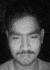}
\hspace{.01in}
\includegraphics{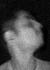}
\hspace{.01in}
\includegraphics{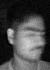}}

\caption{Noise level 1.}
\label{fig:image0p01}
\end{minipage}\hfill
\begin{minipage}{0.48\textwidth}
\centering
\resizebox{5.5cm}{1.13cm}{
\includegraphics{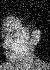}
\hspace{.01in}
\includegraphics{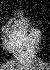}
\hspace{.01in}
\includegraphics{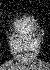}
\hspace{.01in}
\includegraphics{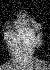}
\hspace{.01in}
\includegraphics{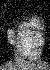}}

\vspace{.05in}

\resizebox{5.5cm}{1.13cm}{
\includegraphics{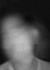}
\hspace{.01in}
\includegraphics{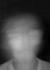}
\hspace{.01in}
\includegraphics{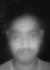}
\hspace{.01in}
\includegraphics{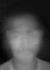}
\hspace{.01in}
\includegraphics{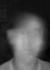}}

\vspace{.05in}

\resizebox{5.5cm}{1.13cm}{
\includegraphics{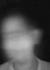}
\hspace{.01in}
\includegraphics{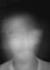}
\hspace{.01in}
\includegraphics{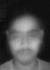}
\hspace{.01in}
\includegraphics{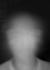}
\hspace{.01in}
\includegraphics{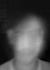}}

\vspace{.05in}

\resizebox{5.5cm}{1.13cm}{
\includegraphics{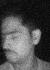}
\hspace{.01in}
\includegraphics{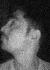}
\hspace{.01in}
\includegraphics{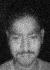}
\hspace{.01in}
\includegraphics{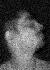}
\hspace{.01in}
\includegraphics{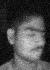}}

\vspace{.05in}

\resizebox{5.5cm}{1.13cm}{
\includegraphics{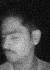}
\hspace{.01in}
\includegraphics{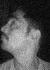}
\hspace{.01in}
\includegraphics{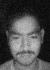}
\hspace{.01in}
\includegraphics{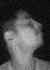}
\hspace{.01in}
\includegraphics{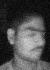}}

\vspace{.05in}

\resizebox{5.5cm}{1.13cm}{
\includegraphics{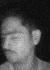}
\hspace{.01in}
\includegraphics{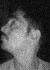}
\hspace{.01in}
\includegraphics{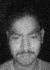}
\hspace{.01in}
\includegraphics{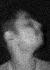}
\hspace{.01in}
\includegraphics{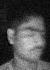}}

\caption{Noise level 2.}
\label{fig:image0p025}
\end{minipage}
\end{figure}

\begin{figure}[h!]
\begin{minipage}{0.48\textwidth}
\centering
\resizebox{5.5cm}{1.13cm}{
\includegraphics{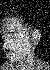}
\hspace{.01in}
\includegraphics{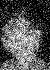}
\hspace{.01in}
\includegraphics{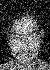}
\hspace{.01in}
\includegraphics{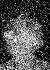}
\hspace{.01in}
\includegraphics{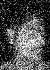}}

\vspace{.05in}

\resizebox{5.5cm}{1.13cm}{
\includegraphics{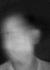}
\hspace{.01in}
\includegraphics{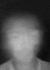}
\hspace{.01in}
\includegraphics{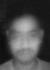}
\hspace{.01in}
\includegraphics{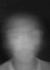}
\hspace{.01in}
\includegraphics{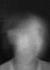}}

\vspace{.05in}

\resizebox{5.5cm}{1.13cm}{
\includegraphics{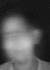}
\hspace{.01in}
\includegraphics{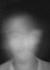}
\hspace{.01in}
\includegraphics{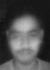}
\hspace{.01in}
\includegraphics{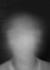}
\hspace{.01in}
\includegraphics{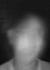}}

\vspace{.05in}

\resizebox{5.5cm}{1.13cm}{
\includegraphics{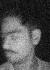}
\hspace{.01in}
\includegraphics{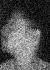}
\hspace{.01in}
\includegraphics{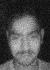}
\hspace{.01in}
\includegraphics{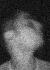}
\hspace{.01in}
\includegraphics{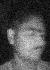}}

\vspace{.05in}

\resizebox{5.5cm}{1.13cm}{
\includegraphics{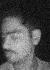}
\hspace{.01in}
\includegraphics{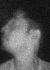}
\hspace{.01in}
\includegraphics{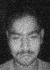}
\hspace{.01in}
\includegraphics{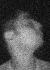}
\hspace{.01in}
\includegraphics{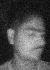}}

\vspace{.05in}

\resizebox{5.5cm}{1.13cm}{
\includegraphics{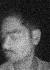}
\hspace{.01in}
\includegraphics{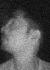}
\hspace{.01in}
\includegraphics{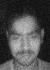}
\hspace{.01in}
\includegraphics{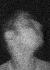}
\hspace{.01in}
\includegraphics{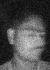}}

\caption{Noise level 3.}
\label{fig:image0p05}
\end{minipage}\hfill
\begin{minipage}{0.48\textwidth}
\centering
\resizebox{5.5cm}{1.13cm}{
\includegraphics{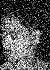}
\hspace{.01in}
\includegraphics{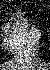}
\hspace{.01in}
\includegraphics{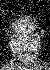}
\hspace{.01in}
\includegraphics{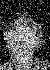}
\hspace{.01in}
\includegraphics{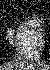}}

\vspace{.05in}

\resizebox{5.5cm}{1.13cm}{
\includegraphics{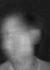}
\hspace{.01in}
\includegraphics{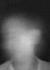}
\hspace{.01in}
\includegraphics{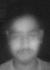}
\hspace{.01in}
\includegraphics{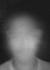}
\hspace{.01in}
\includegraphics{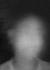}}

\vspace{.05in}

\resizebox{5.5cm}{1.13cm}{
\includegraphics{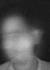}
\hspace{.01in}
\includegraphics{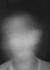}
\hspace{.01in}
\includegraphics{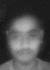}
\hspace{.01in}
\includegraphics{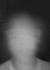}
\hspace{.01in}
\includegraphics{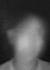}}

\vspace{.05in}

\resizebox{5.5cm}{1.13cm}{
\includegraphics{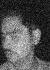}
\hspace{.01in}
\includegraphics{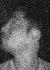}
\hspace{.01in}
\includegraphics{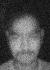}
\hspace{.01in}
\includegraphics{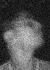}
\hspace{.01in}
\includegraphics{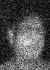}}

\vspace{.05in}

\resizebox{5.5cm}{1.13cm}{
\includegraphics{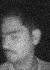}
\hspace{.01in}
\includegraphics{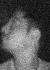}
\hspace{.01in}
\includegraphics{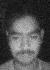}
\hspace{.01in}
\includegraphics{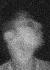}
\hspace{.01in}
\includegraphics{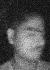}}

\vspace{.05in}

\resizebox{5.5cm}{1.13cm}{
\includegraphics{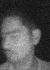}
\hspace{.01in}
\includegraphics{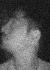}
\hspace{.01in}
\includegraphics{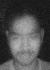}
\hspace{.01in}
\includegraphics{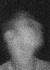}
\hspace{.01in}
\includegraphics{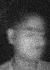}}

\caption{Noise level 4.}
\label{fig:image0p075}
\end{minipage}
\end{figure}

It can be observed that the results from our algorithm are the most visually pleasing, compared with Isomap, global PCA, and local PCA. In addition, the results from Isomap and global PCA show that we were unable to retrieve the main features of the face, even for the sample points with high SNR. On the other hand, local PCA and LLD were able to denoise the images such that the facial features could be retrieved for sample points with a high noise level. For sample points with a low noise level, local PCA and LLD were still able to retrieve some of the features of the face, although the images were noisier than the results from our algorithm.

For the sample points with higher noise levels, our algorithm was still able to denoise the images relatively well, as we were able to retrieve many of the prominent features of the face. 
With a higher level of noise, it naturally became more difficult to denoise the sample points, as indicated by the results for noise level 4. For the higher noise levels, we were unable to retrieve many of the features of the face. However, our algorithm was still able to denoise the sample points minimally, as we were still able to show which direction the face was facing.

These numerical results show that our algorithm is able to work well even with sample points in high-dimensional ambient space, and using only a degree-one polynomial in the MLS step. To achieve better results, we can use a polynomial of a higher degree to make the approximation at the MLS step; however, the trade-off will be higher computational cost.


\section{Conclusion}\label{Sect6}
In this paper, we investigate the problem of manifold fitting. Our focus differs from that of more regular research, which aims to map the input points into a lower-dimensional space to extract the lower-dimensional structure. Instead, we propose to fit a manifold in the ambient space to a set of sample points. A key feature of our approach is the use of the spectrum of the Laplace–Beltrami operator, known as the Shape-DNA, which helps reduce the computational complexity by finding a smaller set of points to operate on, while preserving the geometric structure. This simple and efficient method allows us to approximate a manifold from the sample points without any constraints on the intrinsic dimensions of the underlying manifold. Furthermore, the output of our method is in the original ambient space, which will be more practical for many real-life applications as we will be able to visualize or compare the output in the original space. In our work, we also provide the error bounds for estimating the tangent space using PCA and surface approximation using MLS. The results of our numerical simulations are encouraging. In the simulations for the one-dimensional manifold, the approximated manifold from our method was close to the original manifold. For the face images, the output from our proposed algorithms yielded visually pleasing denoising results for the input with high SNR, and visually acceptable results, compared with PCA, for the input with lower SNR.

We view the subsampling step proposed in this paper as a general step. This subsampling step can also be adapted by other well-established manifold-learning methods to reduce the complexity of their algorithms. As a whole, our algorithm may also be adapted to other dimension-reduction techniques to first denoise the sample points and thereby enable the lower-dimensional structure to be extracted more accurately.

LLD did not work well for the one-dimensional simulation, but seemed to perform better for the image sample points with lower SNR. The reason for this could be the global alignment step involved, which integrates the information of each point from all the local neighborhoods that the point belongs to. A similar global alignment step was also studied by \ctn{LTSA}, who proposed an algorithm for dimension reduction. Therefore, it would be of interest to determine if our manifold-fitting results can be further improved by including the same global alignment step, or whether the global alignment step can be modified to work with our algorithm. These suggestions for the improvement of our proposed algorithm could be considered for future research.

\section*{Acknowledgments}
We are grateful for the financial support from the MOE Tier 1 funding (R-155-000-196-114) and Tier 2 funding (R-155-000-184-112) at the National University of Singapore. 

\appendix
\clearpage
\centerline{\LARGE{\textbf{Supplementary Materials}}}

\section*{S1  Laplace–Beltrami Operator and Heat Operator}
In this section, we provide the definitions for the Laplace–Beltrami operator and heat operator, and discuss their relationship in terms of their spectrum.
\subsection*{S1.1 Laplace–Beltrami Operator}

Let $(\mathcal{M},\bfg)$ be a smooth manifold with an associated metric $\bfg$ embedded in $\mathbb{R}^n$. The Laplace–Beltrami operator $\Delta_\mathcal{M}$ defined on $\mathcal{M}$ is the extension of the Laplace operator $\Delta$ in $\mathbb{R}^n$. In local coordinates, we can write the Laplace–Beltrami operator as
$$\Delta_\mathcal{M} f = \frac{1}{\sqrt{|\det(\bfg)|}}\frac{\partial}{\partial x^i}\Bigg( \sqrt{|\det(\bfg)|}\ g_{ij} \frac{\partial f}{\partial x^j} \Bigg).$$
The Laplace–Beltrami operator admits a spectral decomposition given by the eigenpair $(\lambda_i,\phi_i(x))$, solving the equation
\beq\label{eq:eigenproblem}
\Delta_\mathcal{M} \phi_i(\bfx) = -\lambda_i\phi_i(\bfx),
\eeq
where $\bfx$ is a point on $\mathcal{M}$.

The eigenspectrum of the Laplace–Beltrami operator encompasses many of the important geometric and topological properties of a manifold. Note that the eigenfunctions defined by \eqref{eq:eigenproblem} have been used as a natural basis in shape analysis, which is an analogue to the Fourier basis.

\subsection*{S1.2 Heat Operator}

The heat equation on $\mathcal{M}$ is given by
\beq\label{eq:heateq}
\frac{\partial h(\bfx,t)}{\partial t} = \Delta_\mathcal{M} h(\bfx,t),
\eeq
where $h(\bfx,t)$ is the heat diffusion at time $t$ at point $\bfx$ where $\bfx\in \mathcal{M}$. The heat diffusion $h(x,t)$ describes the flow of heat on the manifold $\mathcal{M}$. Intuitively, we can think of the heat diffusion as a temperature distribution on $\mathcal{M}$. 
\begin{definition*}
[Heat Operator]
The heat operator $H_t$ is defined as the heat distribution at time $t$ of an initial distribution $f(\bfx):=h(\bfx,0)$, i.e. $H_t f(\bfx) := u(\bfx,t).$
\end{definition*}
The heat kernel $K(\bfx,\bfy,t) : \mathcal{M} \times \mathcal{M} \times \mathbb{R}^+  \rightarrow \mathbb{R}^+$ is associated with the heat operator such that
\beq \label{eq:contheat}
H_tf(\bfx)=\int_\mathcal{M}K(\bfx,\bfy,t)f(y)\ dy.
\eeq
If we consider $\mathcal{M}=\mathbb{R}^m$, we have an explicit form for the heat kernel,
\begin{equation*}
K(\bfx,\bfy,t)=\frac{1}{(4\pi t)^{m/2}} \exp\Bigg\{\frac{-\|\bfx-\bfy\|^2}{4t}\Bigg\}.
\end{equation*}

\subsection*{S1.3 Relationship between the Laplace–Beltrami Operator and Heat Operator} \label{sect:LBOnHeat}
The Laplace–Beltrami operator and the heat operator are related such that
$$H_t= \exp\{-t\Delta_\mathcal{M}\}= \sum_{k=0}^{\infty} \frac{(-t\Delta_\mathcal{M})^k}{k!}.$$
Consider the limit as $t\rightarrow 0$; then
$$\Delta_\mathcal{M}=\lim_{t\rightarrow0} \frac{1-H_t}{t}.$$

The eigenfunctions of $H_t$, and hence the eigenfunctions of $\frac{1-H_t}{t}$, coincide with the eigenfunctions of $\Delta_\mathcal{M}$. Let $\lambda_i$ be the $i$-th eigenvalue of $\Delta_\mathcal{M}$; then, the $i$-th eigenvalue of $\frac{1-H_t}{t}$ is $\frac{1-\exp\{-t\lambda_i\}}{t}$. When an invariant metric is chosen, the Laplace–Beltrami operator and the corresponding heat operator admit invariant properties. Given a set of sample points, it is not practical to calculate the Laplace–Beltrami operator on a manifold; instead, we work with the heat operator. The heat kernel can be computed using eigenpairs of the Laplace–Beltrami operator:
\beq \label{eq:heatkernel}
K(\bfx,\bfy,t)=\sum_{i=0}^{\infty}e^{-\lambda_i t}\phi_i(\bfx)\phi_i(\bfy).
\eeq

In a continuous setting, the heat equation is given by \eqref{eq:contheat}. Given a manifold with a finite sample-point set $X=\{ \bfx_1,\ldots,\bfx_k\}$, we can approximate the heat operator as a matrix, $\mathbf{H}_t$, where $(\mathbf{H}_t)_{ij}=K(\bfx_i,\bfx_j,t)$. Let $\mathbf{f}=[f(\bfx_1),\ldots,f(\bfx_k)]^T$. Then, the discrete heat equation is 
\beq\label{eq:disheateq}
(\mathbf{H}_t\mathbf{f})_i=\sum_{j=1}^k K(\bfx_i,\bfx_j,t) \mathbf{f}_j.
\eeq
Here, as $k\rightarrow \infty$, \eqref{eq:disheateq} will converge to \eqref{eq:contheat}. In this paper, we deal with finite sample points, and thus will be using $\mathbf{H}_t$ to approximate the heat operator in our algorithm.
\section*{S2 Derivation of conditions (4.8) and (4.9)}
\setcounter{equation}{0}

In this section, we derive the conditions (4.8) and (4.9), which are required for the right-hand side of the error bound (4.7) to be non-negative.

We note that, for the right-hand side of the error bound (4.7), the numerator is always positive regardless of the $r$ that we choose. Thus we only require the denominator to be positive, i.e.
$$
\frac{r^2}{m+2} - \frac{\mathcal{K}r^4}{2(m+2)^2(m+4)}-\sigma^2\big(\sqrt{m}+\sqrt{n-m}\big) > 0 .
$$
We also note that the denominator, $\frac{r^2}{m+2} - \frac{\mathcal{K}r^4}{2(m+2)^2(m+4)}-\sigma^2\big(\sqrt{m}+\sqrt{n-m}\big)$, with respect to $r^2$ is a quadratic function of the form $ax^2+bx+c$, where $a$ is negative.

For a quadratic function to be positive, i.e. $ax^2+bx+c>0$, with $a<0$, we first require it to have two real roots; the function will be positive between the roots. Thus we arrive at the condition $b^2-4ac>0$, and the function will be positive when 
$$x\in \Bigg( \min\Bigg(\frac{-b\pm\sqrt{b^2-4ac}}{2a}\Bigg) , \max\Bigg(\frac{-b\pm\sqrt{b^2-4ac}}{2a}\Bigg) \Bigg).$$
Substituting 
$$a=-\frac{\mathcal{K}}{2(m+2)^2(m+4)},\quad b=\frac{1}{m+2}\quad \text{and}\quad c=-\sigma^2\big(\sqrt{m}+\sqrt{n-m}\big),$$
we arrive at conditions (4.8) and (4.9).
%
%
\section*{S3 Proof for Theorem 1}
\setcounter{equation}{0}

In this section, we provide the proof for Theorem 1.

Recall that the standard basis, $\{p_{\alpha}\}_{|\alpha|\leq m}$, defined for $\Pi_q^m$ is shifted by the reference point $\bfx$. For simplicity, we assume that there are $J$ such basis functions and $p_1$ corresponds to the constant basis function, i.e. $p_1(\bfx) = \bar 1$. Then, $\hat f_q=\sum_{i=1}^Jc_ip_i$, and we have
\begin{align*}
\hat f_q(\bfx)&=\sum_{i=1}^Jc_{i}p_{i}(\bfx)\\
&=c_1 + \sum_{0<|\alpha|\leq m}c_{\alpha}(\bfx - \bfx)^{\alpha}\\
&=c_1.
\end{align*}
The coefficient, $\bfc=[c_1,\cdots,c_J]^T$, for the approximated polynomial, $\hat f_q$, is defined as the solution for
\beq\label{eq:MLS}
\bfV^T\bfW\bfV\bfc=\bfV^T\bfW(\bfF+\bfe),
\eeq
where $\bfF=[f(\bfx_1),\cdots,f(\bfx_I)]^T$ and $\bfe=[e_1,\cdots,e_I]^T$. Using Taylor expansion and the mean value theorem, we have
$$f(\bfx_i) = f(\bfx) + \sum_{0<|\alpha|\leq m}\frac{D^{\alpha}f(\bfx)}{\alpha!}(\bfx_i - \bfx)^\alpha + \frac{1}{(m+1)!}\sum_{|\alpha|=m+1}D^\alpha f(\pmb{\epsilon}_i)(\bfx_i - \bfx)^\alpha,$$
where $\pmb{\epsilon}_i = (1-a)\bfx_i + a\bfx,$ for $0\leq a\leq 1$. Then,
$$\bfF = f(\bfx)\bfV_1 + \sum_{0<|\alpha|\leq m}\frac{D^{\alpha}f(\bfx)}{\alpha!}\bfV_\alpha + \frac{1}{(m+1)!}\sum_{|\alpha|=m+1}\bfQ_\alpha \bfV_\alpha,$$
where $\bfQ_\alpha=diag\big(D^\alpha f(\pmb{\epsilon}_1), \cdots,D^\alpha f(\pmb{\epsilon}_I)\big)$ and $\bfV_\alpha$ represents the $\alpha$ column of $\bfV$. $\bfV_1$ is a column of 1's with dimension $(1\times I)$, which is also the first column of $\bfV$. From \eqref{eq:MLS}, we get 
\begin{align*}
c_1&=f(\bfx) + \frac{1}{(m+1)!}\Bigg(\sum_{|\alpha|=m+1}(\bfV^T\bfW\bfV)^{-1}\bfV^T\bfW\bfQ_\alpha \bfV_\alpha\Bigg)_1\\ 
&\qquad\quad + \Bigg((\bfV^T\bfW\bfV)^{-1}\bfV^T\bfW\bfe\Bigg)_1\\
&=f(\bfx) + \frac{1}{(m+1)!}\Bigg(\sum_{|\alpha|=m+1}(\bfV^T\bfW\bfV)^{-1}\sum_{i=1}^I (\bfV^T\bfW)_iD^\alpha f(\pmb{\epsilon}_i)(\bfx_i - \bfx)^\alpha\Bigg)_1\\
&\qquad\quad+ \Bigg((\bfV^T\bfW\bfV)^{-1}\sum_{i=1}^I (\bfV^T\bfW)_i e_i\Bigg)_1.
\end{align*}
Following Cramer's rule, we have
$$\Big((\bfV^T\bfW\bfV)^{-1} (\bfV^T\bfW)_i\Big)_1 = \frac{\det ((\bfV^T\bfW\bfV)_{1\leftarrow  (\bfV^T\bfW)_i})}{\det (\bfV^T\bfW\bfV)}.$$
Therefore,
\begin{align*}
c_1&=f(\bfx) + \frac{1}{(m+1)!}\sum_{|\alpha|=m+1}\sum_{i=1}^I \frac{\det ((\bfV^T\bfW\bfV)_{1\leftarrow  (\bfV^T\bfW)_i})}{\det (\bfV^T\bfW\bfV)}D^\alpha f(\pmb\epsilon_i)(\bfx_i -\bfx)^\alpha\\
&\qquad \quad +\sum_{i=1}^I \frac{\det ((\bfV^T\bfW\bfV)_{1\leftarrow  (\bfV^T\bfW)_i})}{\det (\bfV^T\bfW\bfV)}e_i
\end{align*}
and
\begin{align*}
\hat f_q(\bfx)-f(\bfx)&= c_1-f(\bfx)\\
&= \frac{1}{(m+1)!}\sum_{|\alpha|=m+1}\sum_{i=1}^I \frac{\det ((\bfV^T\bfW\bfV)_{1\leftarrow  (\bfV^T\bfW)_i})}{\det (\bfV^T\bfW\bfV)}D^\alpha f(\pmb\epsilon_i)(\bfx_i - \bfx)^\alpha\\
&\qquad +\sum_{i=1}^I \frac{\det ((\bfV^T\bfW\bfV)_{1\leftarrow  (\bfV^T\bfW)_i})}{\det (\bfV^T\bfW\bfV)}e_i\\
|\hat f_q(\bfx)-f(\bfx)|&\leq\Bigg| \frac{C_{\alpha}}{(m+1)!}\sum_{|\alpha|=m+1}\sum_{i=1}^I \frac{\det ((\bfV^T\bfW\bfV)_{1\leftarrow  (\bfV^T\bfW)_i})}{\det (\bfV^T\bfW\bfV)}(\bfx_i - \bfx)^{\alpha} \Bigg|\\
&\qquad +\Bigg|\sum_{i=1}^I \frac{\det ((\bfV^T\bfW\bfV)_{1\leftarrow  (\bfV^T\bfW)_i})}{\det (\bfV^T\bfW\bfV)}e_i\Bigg|.
\end{align*}
Let $C=\max_{i=1}^I \frac{\det ((\bfV^T\bfW\bfV)_{1\leftarrow  (\bfV^T\bfW)_i})}{\det (\bfV^T\bfW\bfV)}$. We then see that
$$\Bigg| \frac{C_{\alpha}}{(m+1)!}\sum_{|\alpha|=m+1}\sum_{i=1}^I \frac{\det ((\bfV^T\bfW\bfV)_{1\leftarrow  (\bfV^T\bfW)_i})}{\det (\bfV^T\bfW\bfV)}(\bfx_i - \bfx)^{\alpha} \Bigg| = O(r^{m+1}),$$
and 
$$|\hat f_q(\bfx)-f(\bfx)| \leq O(r^{m+1}) + \Bigg|C\sum_{i=1}^I e_i\Bigg|.$$
We know that the sum of normally distributed random variables is also a random variable, i.e. $\sum_{i=1}^I e_i \sim \mathcal{N}(0,\sigma^2\bfI)$. We get $\frac{\sum_{i=1}^I e_i}{\sigma\sqrt{I}} \sim N(0,1)$ and
\begin{align*}
Pr\Bigg( \Bigg|\frac{\sum_{i=1}^I e_i}{\sigma\sqrt{I}}\Bigg| > |k|  \Bigg) &\leq \frac{2}{\sqrt{2\pi}k}e^{-\frac{k^2}{2}},\\
Pr\Bigg( \Bigg|\sum_{i=1}^I e_i \Bigg| > |k\sigma\sqrt{I}|  \Bigg) &\leq \frac{2}{\sqrt{2\pi}k}e^{-\frac{k^2}{2}},\\
Pr\Bigg( \Bigg|\sum_{i=1}^I e_i \Bigg| \leq |k\sigma\sqrt{I}|  \Bigg) &\geq 1- \frac{2}{\sqrt{2\pi}k}e^{-\frac{k^2}{2}}.
\end{align*}
Therefore,
\begin{align*}
|\hat f_q(\bfx)-f(\bfx)| &\leq O(r^{m+1}) + |Ck\sigma\sqrt{I}|,
\end{align*}
with probability at least $1- \frac{2}{\sqrt{2\pi}k}\exp\{-\frac{k^2}{2}\}$. \qed

\section*{S4 Additional Numerical Simulations}

In this section, we provide additional numerical simulations to show the stability of our algorithm and provide additional comparisons with the other methods.

\subsection*{S4.1 The Ablation Experiment for Six-folded and Swiss roll}
In this subsection, we provide the results of ablation experiments on the low-dimensional manifolds six-folded and Swiss roll. We compare the two algorithms (MLS algorithm with subsampling and without subsampling), with the results shown in Figures \ref{fig:ablationsixfold} and \ref{fig:ablationswiss}. It can be observed from Figure \ref{fig:ablationsixfold} that MLS without subsampling yields only discrete data points rather than smooth manifolds. When the curvature is large, the error of MLS without subsampling is also large. In contrast, MLS with subsampling produces smooth curves and better approximation results for sample points with larger curvature. The results of the Swiss roll are very similar to those of six-folded.

\begin{figure}[h]
    \includegraphics[width = 1\linewidth, height = 0.7\linewidth]{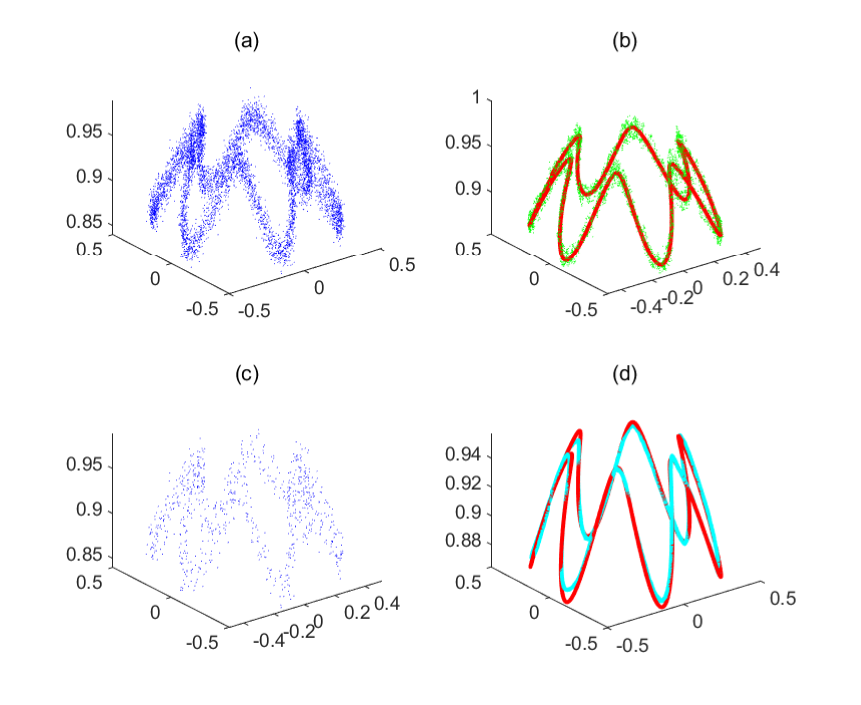}        
    \caption{Ablation experiments for subsampling on six-folded. (a) Sample points with noise. (b) The result of the MLS algorithm without subsampling (green) and the noiseless sample points (red). (c) The anchor points obtained by subsampling. (d) The result of the MLS algorithm with subsampling (teal) and the noiseless sample points (red).}
    \label{fig:ablationsixfold}
\end{figure}

\begin{figure}[h]
    \includegraphics[width = 1\linewidth, height = 0.28\linewidth]{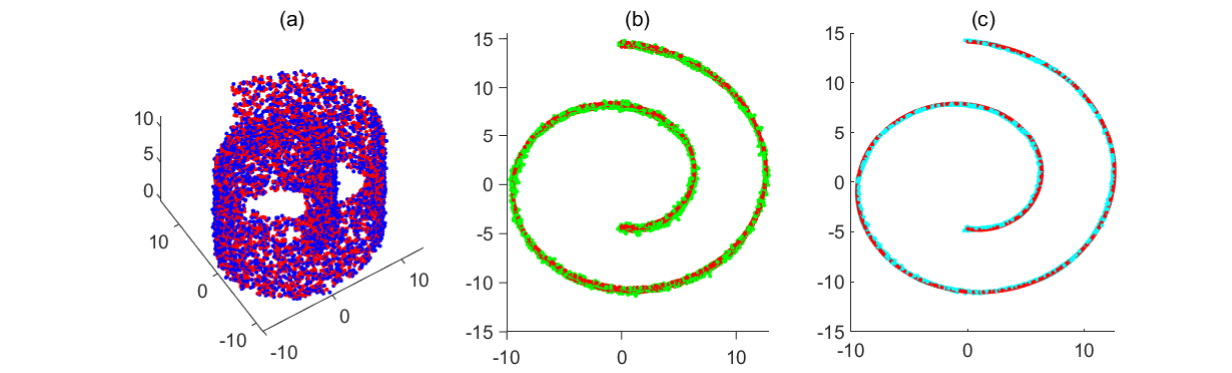}        
    \caption{Ablation experiments for subsampling on Swiss roll. (a) Sample points with noise. (b) The result of the MLS algorithm without subsampling (green) and the noiseless sample points (red). (c)  The result of the MLS algorithm with subsampling (teal) and the noiseless sample points (red).}
    \label{fig:ablationswiss}
\end{figure}

\subsection*{S4.2 Image Data Simulation}
For this simulation, we worked with the same sample points as in the Image Data Simulation 2 of the main paper, but added a different type of noise to only 15 of the frames. We chose the same five frames as in Figure 6 of the main paper and, for each of the five frames, selected the two nearest frames in terms of the Euclidean distance. For each of the 15 frames, we added Gaussian noise in a narrow horizontal band in the middle of the image. Our algorithm was then compared with LLD, and the results are shown in Figure \ref{fig:5n1}.
\begin{figure}[h!]
\centering
\resizebox{5.5cm}{1.13cm}{
\includegraphics{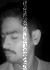}
\hspace{.01in}
\includegraphics{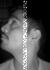}
\hspace{.01in}
\includegraphics{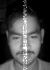}
\hspace{.01in}
\includegraphics{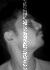}
\hspace{.01in}
\includegraphics{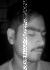}}

\vspace{.05in}

\resizebox{5.5cm}{1.13cm}{
\includegraphics{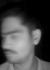}
\hspace{.01in}
\includegraphics{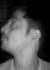}
\hspace{.01in}
\includegraphics{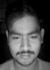}
\hspace{.01in}
\includegraphics{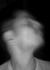}
\hspace{.01in}
\includegraphics{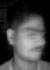}}

\vspace{.05in}

\resizebox{5.5cm}{1.13cm}{
\includegraphics{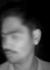}
\hspace{.01in}
\includegraphics{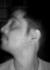}
\hspace{.01in}
\includegraphics{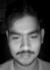}
\hspace{.01in}
\includegraphics{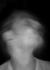}
\hspace{.01in}
\includegraphics{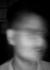}}

\caption{The first, second, and third rows display the noisy images, the results from LLD, and those from our algorithm, respectively.}
\label{fig:5n1}
\end{figure}

For this simulation, we wanted to determine whether, if only a handful of sample points were corrupted in such a way that most of the image remained noiseless, the manifold-learning algorithms would be able to denoise the images much better, or would corrupt the images more instead. Here, as the majority of the images were noiseless, we chose to use 5 principal components and 200 nearest neighbors for LLD and tangent-space approximation to simulate the algorithms. The results show that both our method and LLD were able to almost fully denoise the image, with the results of our method being better. However, there was also some blurring of the facial features, which is noticeable in the fourth and fifth images for both our method and LLD. This was the result of a smaller number of principal components being used.

\subsection*{S4.3 Radius test}\label{sect:radius_test}
For this simulation, we worked with the same sample points of noise level 1 from Subsection 5.3 of the main paper. However, we applied our algorithm using 150 nearest neighbors, and then repeated the algorithm twice, with 75 and 225 nearest neighbors, to compare the effects of using different numbers of nearest neighbors.
\begin{figure}[h!]
\centering
\raisebox{\dimexpr 0.56cm-0.5\height}{\phantom{1}75 nearest neighbors:}\hspace{\bibindent}
\resizebox{5.5cm}{1.13cm}{
\includegraphics{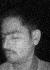}
\hspace{.01in}
\includegraphics{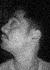}
\hspace{.01in}
\includegraphics{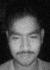}
\hspace{.01in}
\includegraphics{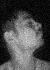}
\hspace{.01in}
\includegraphics{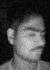}}

\vspace{.05in}
\raisebox{\dimexpr 0.56cm-0.5\height}{150 nearest neighbors:}\hspace{\bibindent}
\resizebox{5.5cm}{1.13cm}{
\includegraphics{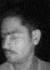}
\hspace{.01in}
\includegraphics{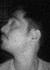}
\hspace{.01in}
\includegraphics{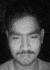}
\hspace{.01in}
\includegraphics{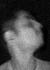}
\hspace{.01in}
\includegraphics{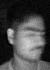}}

\vspace{.05in}

\raisebox{\dimexpr 0.56cm-0.5\height}{225 nearest neighbors:}\hspace{\bibindent}
\resizebox{5.5cm}{1.13cm}{
\includegraphics{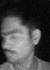}
\hspace{.01in}
\includegraphics{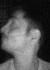}
\hspace{.01in}
\includegraphics{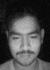}
\hspace{.01in}
\includegraphics{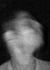}
\hspace{.01in}
\includegraphics{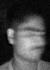}}

\caption{Comparison between experiments with different numbers of nearest neighbors.}
\label{fig:4n1}
\end{figure}

The results in Figure \ref{fig:4n1} show that, with 75 nearest neighbors, the output still contains a substantial amount of noise, compared with the output with 150 or 225 nearest neighbors. This supports the theoretical results in Subsection 4.1 of the main paper, and indicates that the radius cannot be too small, or there will be a larger error.

The output with 150 nearest neighbors shows clearer facial features than that with 225 nearest neighbors. If we had included more neighborhood points, the neighborhood radius would have increased, which would have had a negative effect on the accuracy; this is also indicated by the theoretical results of Subsection 4.1 of the main paper, which suggests that a larger radius will cause the curvature to have a larger effect on the approximation.

\subsection*{S4.4 Stability Test}\label{sect:box1}
We simulated multiple different sets of sample points for the helix, six-folded curve and Swiss roll using different Gaussian noise (with the same Gaussian-noise-parameter settings) for each set of sample points. For each type of manifold, we had 50 simulations, and we plotted the Hausdorff error for each simulation using a boxplot, as shown in Figure \ref{fig:boxplot}.

\begin{figure}[h!]
\includegraphics[height=1.2in,width=1.7in]{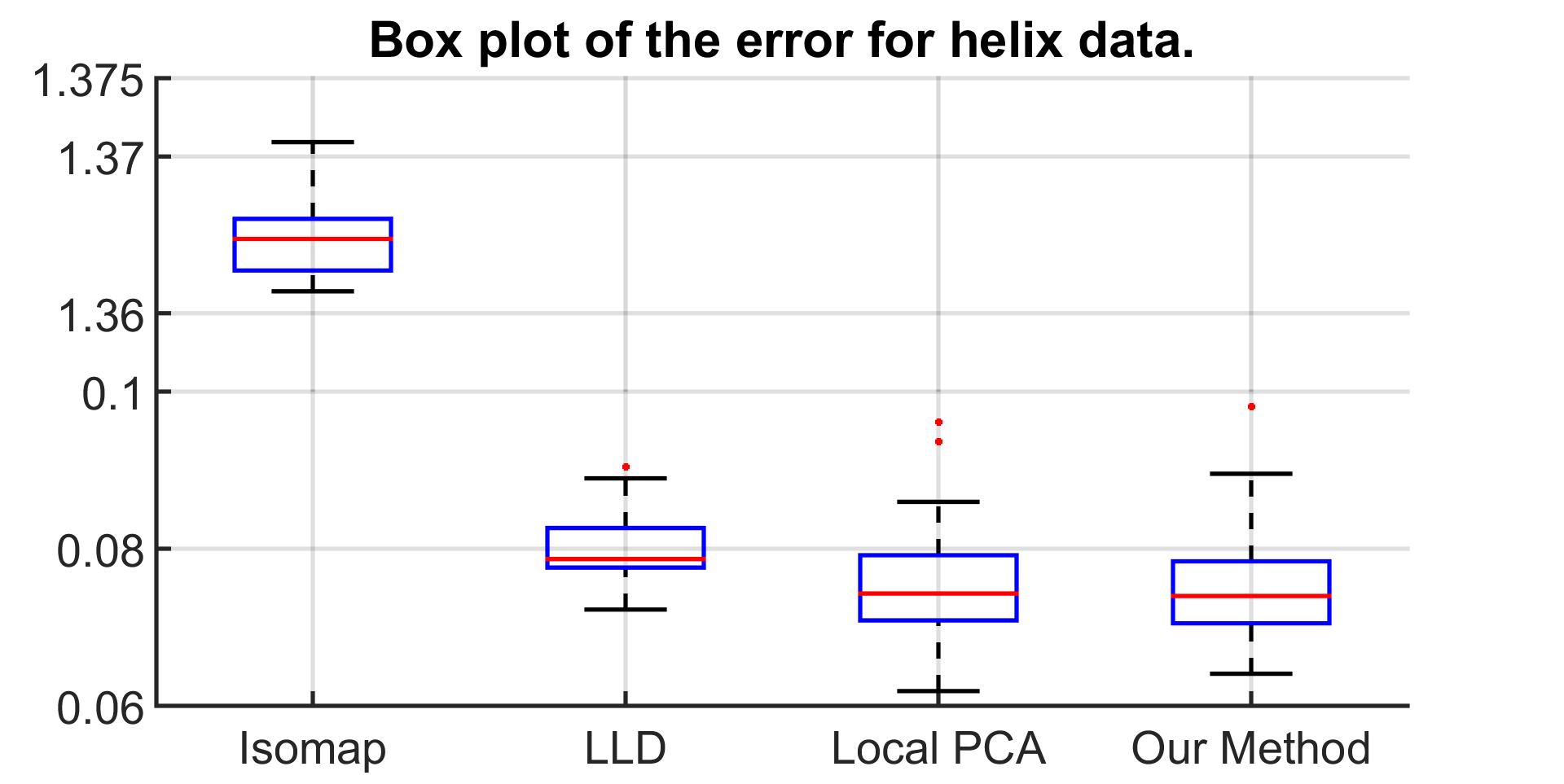}
\includegraphics[height=1.2in,width=1.7in]{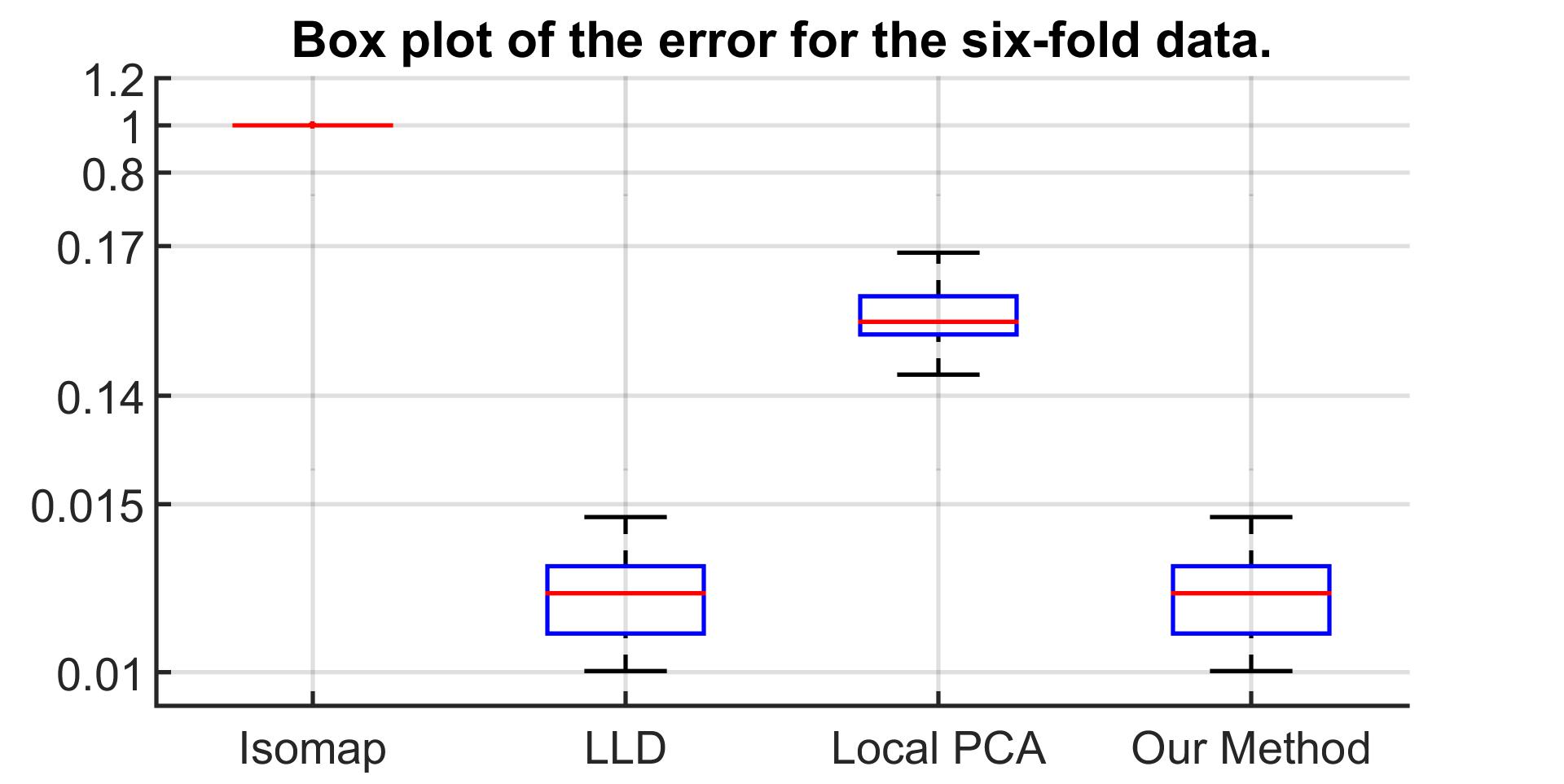}
\includegraphics[height=1.2in,width=1.7in]{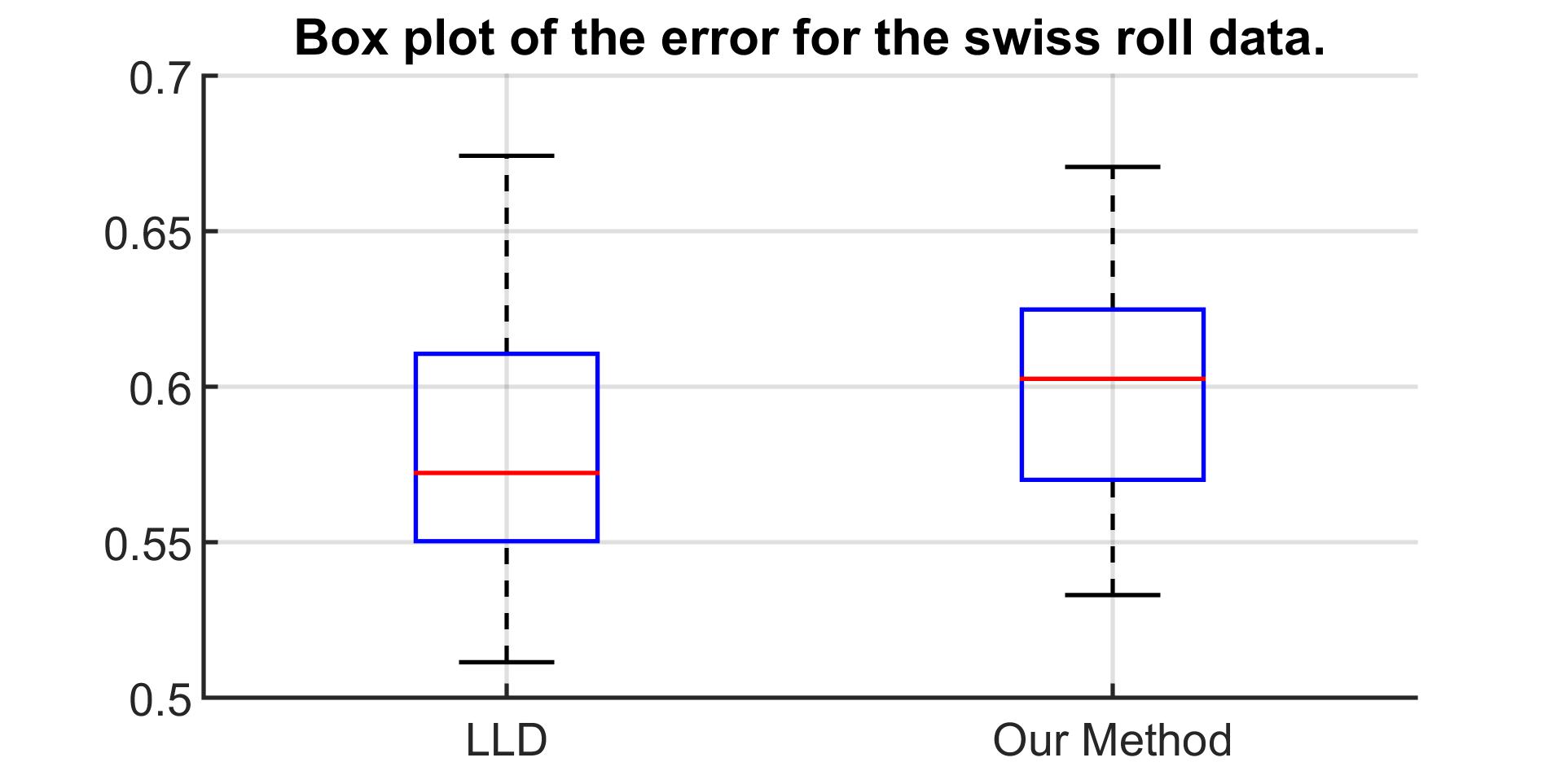}
\caption{Boxplot of Hausdorff error.}
\label{fig:boxplot}
\end{figure}

Figure \ref{fig:boxplot} shows that our algorithm is stable and the Hausdorff error for our method is close to that of LLD. Furthermore, our method performs better than local PCA and Isomap, as discussed in the main paper.

\section*{S5 Eigenspectrum for Laplace–Beltrami operator}
We compare two sets of sample points on a unit sphere, with one set containing 624 points and the other containing 2526 points. The exact eigenvalue of the Laplace–Beltrami operator for the unit sphere is $l(l+1)$ with multiplicity $2l+1$ for $l=0,1,2,\ldots$. Here, there is an infinite number of eigenvalues when we consider the manifold in a continuous sense. However, when we consider a set of sampled points of the manifold, we are only able to approximate a finite number of the eigenvalues equal to the number of sampled points. We approximate the eigenvalues using the algorithm from \ctn{shapeDNA} for both sets. The difference, calculated by $\Big|\frac{\lambda_{i,624}-\lambda_{i,2526}}{\lambda_{i,\text{exact}}}\Big|$, where $\lambda_{i,624}$ and $\lambda_{i,2526}$ are the $i$-th smallest eigenvalues for the sets containing 624 and 2526 points, respectively, and $\lambda_{i,exact}$ is the exact eigenvalue, is plotted in Figure \ref{fig::sphereeigen}. Here, it can be observed that, although the normalized difference in eigenvalues increases, it still remains quite small even at the $150^{\text{th}}$ eigenvalue. The increase in the difference is to be expected as we are using a smaller number of points to approximate a continuous manifold, i.e. 624 as opposed to 2526 points. However, the benefit is that, with a smaller number of points, we are still able to extract the structure of the underlying manifold and achieve a lower cost in terms of computational complexity.

\begin{figure}[h!]
\begin{center}
\includegraphics[width=1.2in]{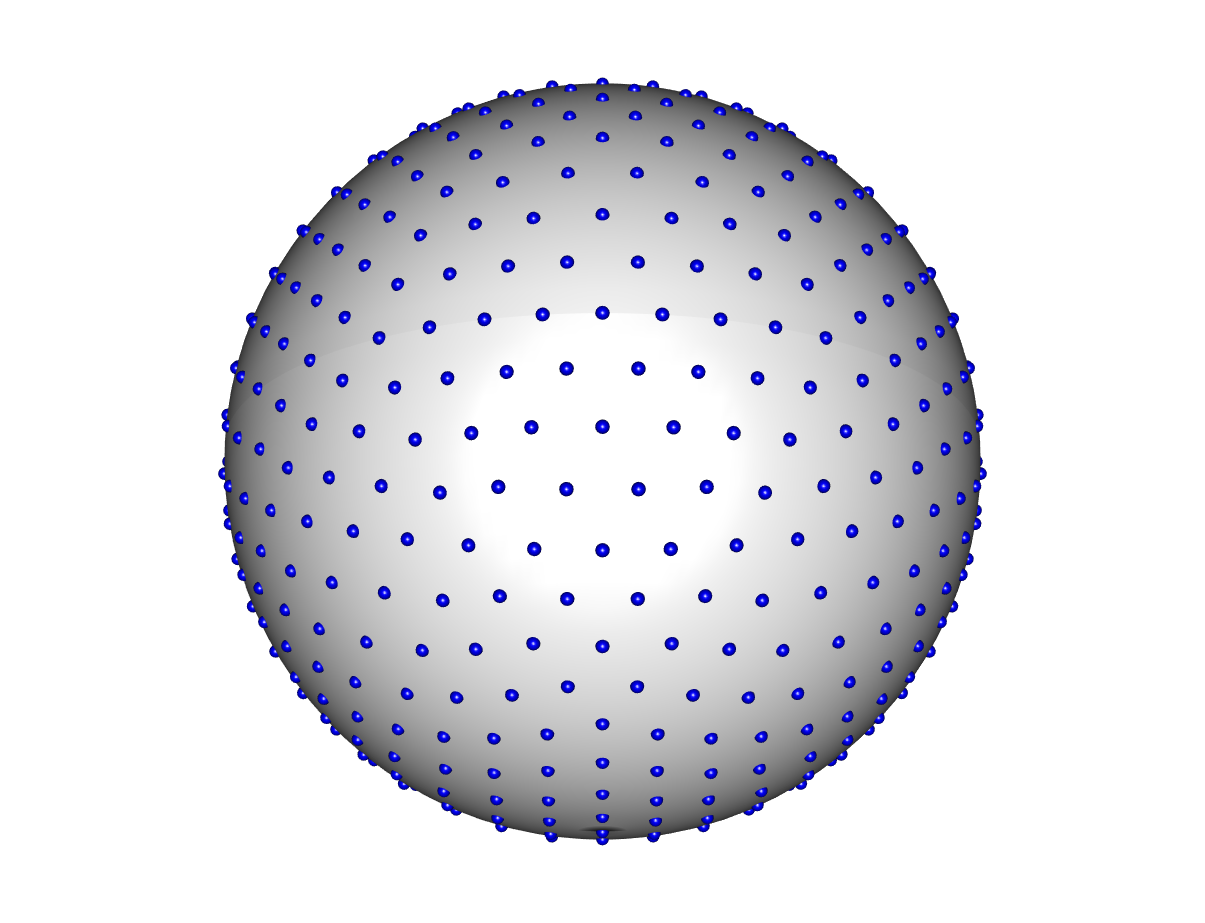}
\includegraphics[width=1.2in]{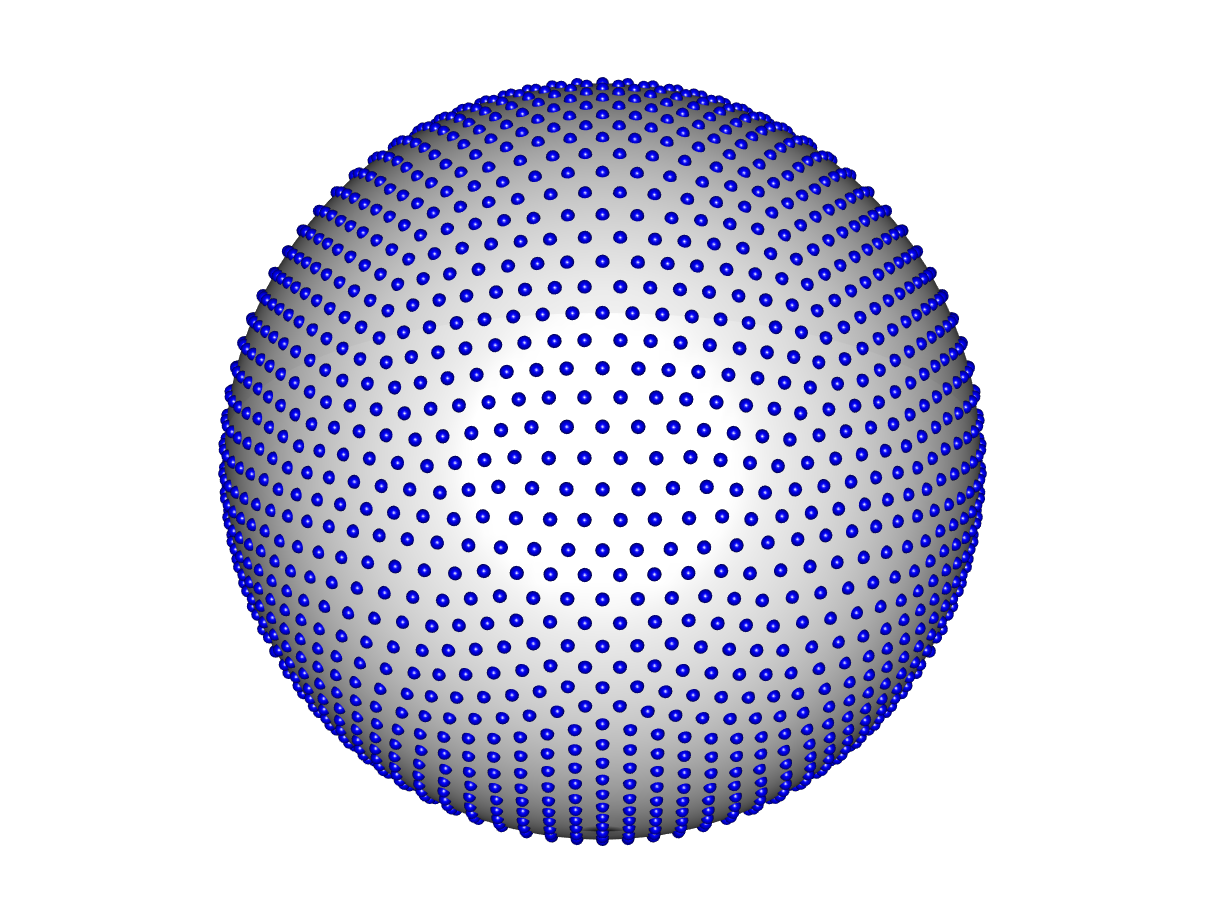}
\includegraphics[width=2.4in]{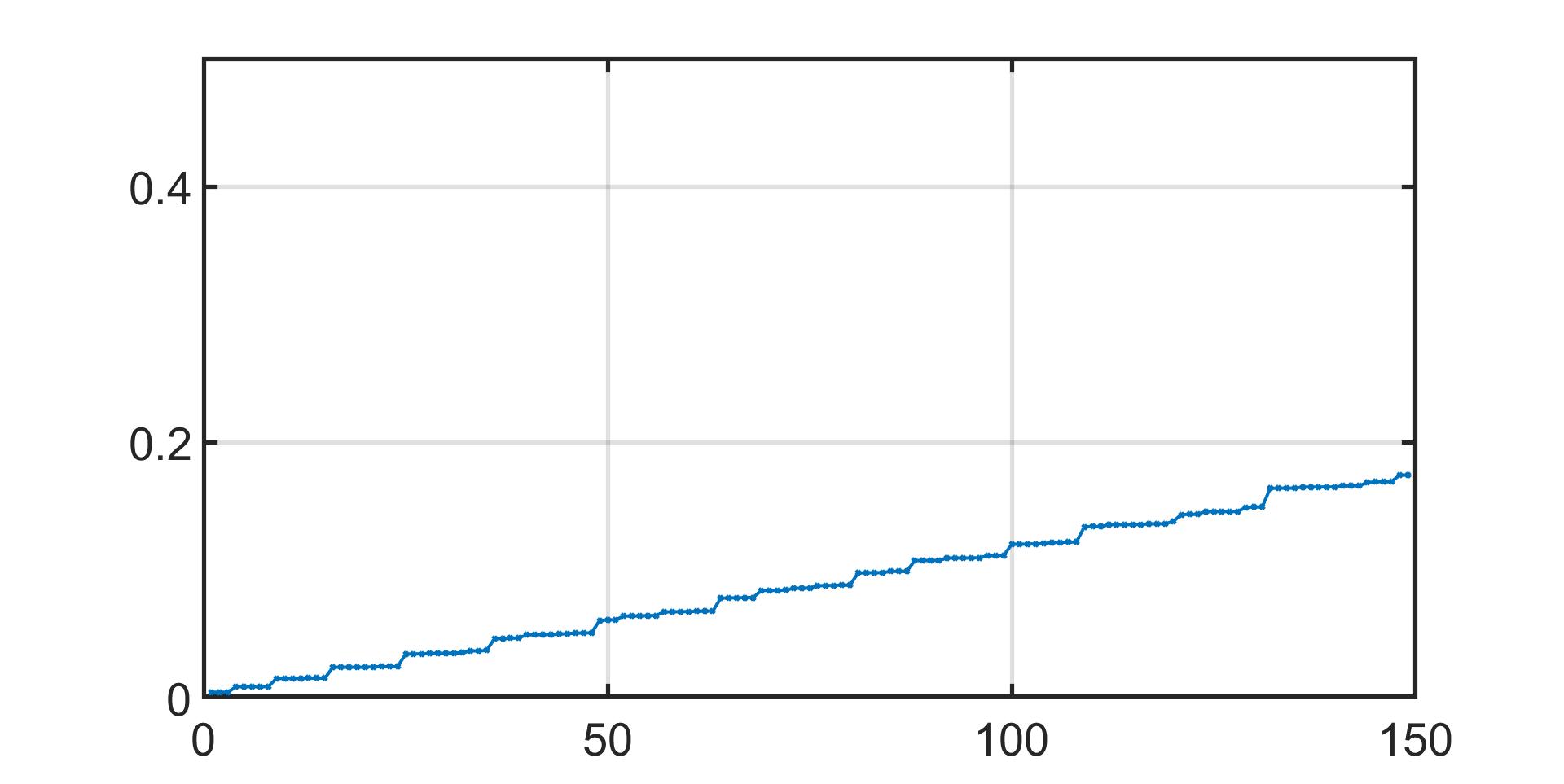}
\end{center}
\caption{Comparison of the eigenvalues of the Laplace–Beltrami operator.}
\label{fig::sphereeigen}
\end{figure}

By using the eigenvalues of the Laplace–Beltrami operator, the subsampled set still preserves the geometry of the underlying manifold, as demonstrated by the example of the unit sphere. Therefore, by working with the subsampled set, we can reduce the complexity of the computations and still retrieve the geometry of the underlying manifold.

\bibliographystyle{abbrv}
\bibliography{refer}

\end{document}